\documentclass[11pt]{article}
\usepackage[margin=1in]{geometry}

\usepackage{amsthm}
\usepackage{amsmath}
\usepackage{amstext}
\usepackage{amssymb}

\usepackage{amsmath,amsfonts,bm}

\def\eqref#1{equation~\ref{#1}}
\def\Eqref#1{Equation~\ref{#1}}

\def\1{\bm{1}}

\DeclareMathAlphabet{\mathsfit}{\encodingdefault}{\sfdefault}{m}{sl}
\SetMathAlphabet{\mathsfit}{bold}{\encodingdefault}{\sfdefault}{bx}{n}

\newtheorem{theorem}{Theorem}
\newtheorem{definition}{Definition}
\newtheorem{proposition}{Proposition}
\newtheorem{corollary}{Corollary}
\newtheorem{lemma}{Lemma}
\newtheorem{remark}{Remark}
\newtheorem{assumption}{Assumption}
\usepackage{hyperref}
\usepackage[ruled]{algorithm}
\usepackage{mathrsfs}
\usepackage{algpseudocode}
\usepackage{graphicx}
\usepackage{etoolbox}
\usepackage{xcolor}
\usepackage{url}
\usepackage{natbib}
\date{}
\title{Conditioning Degenerate Diffusion Models}

\author{U\u{g}ur Ayd{\i}n \& Tamer Ba{\c{s}}ar \\
Department of Electrical and Computer Engineering and Coordinated Science Laboratory\\
University of Illinois Urbana Champaign\\
Urbana, IL 61801, USA \\
\texttt{\{uaydin2, basar1\}@illinois.edu}
}

\begin{document}

\maketitle

\begin{abstract}

Current conditioned generative models heavily rely on score functions for guidance during training. When the generative model is a diffusion process with a singular diffusion coefficient and the underlying (conditional) densities either do not exist or are not smooth, we use causal optimal transport to define \emph{approximate} loss functions that identify a minimum-entropy control for guidance under minimal assumptions. Our approach relies on causal optimal transport and its characterization through the predictable representation property of (conditioned) diffusion processes whose associated martingale problem is well posed, \`a la \"Ust\"unel.
\end{abstract}

\section{Introduction}

Conditioned diffusion processes appear in physical chemistry
\cite{bolhuis2002transition}, genetics \cite{wang2011quantifying}, economics
\cite{elerian2001likelihood}, and guided generative models
\cite{zhang2023adding, domingo2025adjoint}. In practice, a common feature of these problems is that the
conditioned process must be constructed from data generated by an
unconditioned process $X=(X_t)_{0\leq t\leq T}$ together with a terminal condition \cite{fabrice,tzen2019theoretical,pidstrigach2025conditioning}.  Hence,
at time $t$, the available data are the state history and the condition.

In particular, consider an $n$-dimensional stochastic differential equation (SDE) of the
form
\begin{equation}\label{eq:s}
dX_t=b(t,X_t)\,dt+\sigma(t,X_t)\,dB_t,
\qquad X_0=\eta,
\end{equation}
where $B$ is a $d$-dimensional Brownian motion, and a terminal condition
$Y=G(X_T)$, where $G:\mathbb R^n\to\mathbb R^k$ and $\eta$ is potentially a random variable.  Degeneracy refers to the case in which
$\sigma(t,x)\in\mathbb R^{n\times d}$ does not have full rank (which can be also extended to the case where the corresponding Malliavin covariance operator is also singular).  When the
conditional likelihood is sufficiently smooth, the Doob $h$-transform
identifies a controlled equation
\begin{equation}\label{eq:s2}
dX_t^u
=\bigl(b(t,X_t^u)+\sigma(t,X_t^u)u_t(X_t^u,y)\bigr)\,dt
+\sigma(t,X_t^u)\,dB_t,
\qquad X_0^u\sim\mathcal L(\eta\mid Y=y),
\end{equation}
whose state-path law agrees with $\mathcal L(X_{[0,T]}\mid Y=y)$
\cite{pidstrigach2025conditioning,tzen2019theoretical}.  If
$Z_t^y=h_t(X_t;y)$ is a smooth conditional likelihood, then the
Brownian-space shift is
$u_t=\sigma(t,X_t)^\top\nabla_x\log h_t(X_t;y)$, and the added state drift
is $\sigma u=\sigma\sigma^\top\nabla_x\log h$.  This shift is the F\"ollmer
drift of the corresponding Brownian realization, and it is characterized by
a minimum-entropy property \cite{lassalle2018causal}, which also has connections with 
causal optimal transport.

When $\sigma$ is not invertible, this problem is considerably more difficult than in the invertible case. First, the filtration $\mathcal F(X)$ need not coincide with that of the driving (ambient) Brownian motion $B$. As suggested by the notation, we require the \emph{control} $(u_t(X^u_t,y))_t$ to be $\mathcal F(X^u)$-adapted, i.e. \emph{a closed-loop control}. This is important because adaptiveness of $u$ to the path of $X^u$ is crucial for the controlled SDE \eqref{eq:s2} to remain \emph{closed-loop}; otherwise, its implementation would require information external to $\mathcal F(X)\vee\sigma(Y)$ generated from \eqref{eq:s}. Furthermore, when $G$ is the identity map, as is the case, for instance, when \eqref{eq:s} is used as a bridge model, the corresponding conditional law at the terminal time is a Dirac measure and is therefore singular rather than smooth.

\section{Related Work and Contributions}

\subsection{Related Work}

In this work, we introduce loss functions that yield a control $u$ such that $X$ and $X^u$, introduced in \eqref{eq:s} and \eqref{eq:s2}, respectively, have the same law. This result relies on the adapted representation property of SDEs satisfying weak uniqueness \cite{ust_mart}, which was first established in \cite{stroock1980extremal}; see also \cite{elworthy2010geometry} for a special case on Riemannian manifolds. Using this property, \cite{ust_mart} also characterized the associated causal optimal transport problem. We use these results to derive a \emph{closed-loop} variant of the so-called F\"ollmer drift $(u_t)_t$ under minimal conditions, which has minimum energy among the admissible shifts (that satisfy Girsanov's theorem). Moreover, through a localization argument, we relax the strict positivity condition imposed in \cite{ust_mart} in their treatment of the causal optimal transport problem.

With the advancements of generative models, there has been recent interest on the behavior of degenerate diffusion models. The works \cite{dockhorn2021score, blessing2503underdamped} study bridge models under degeneracy when the densities are sufficiently smooth. A theory for manifold-based generative models has been developed in \cite{debortoli2022convergence, de2022riemannian}, which are inherently degenerate and satisfy less restrictive smoothness conditions on the corresponding densities compared to their Euclidean variants.

In the presence of degeneracy or non-smooth densities, recently Malliavin calculus based approaches have been used \cite{pidstrigach2025conditioning, mirafzali2025malliavin}. However, Malliavin calculus based approaches still require sufficient smoothness and regularity (as in non-degeneracy of the corresponding Malliavin covariance operator) of $b$ and $\sigma$. Several different approaches have also appeared \cite{conforti2025kl, blessing2503underdamped}. However, all these works either assume that $\mathcal L(X_t |Y=y)$ is continuously differentiable, or suppose that $b$ and $\sigma$ are sufficiently smooth/regular.

\subsection{Contributions}

Our contributions in this paper are summarized as follows.
\begin{enumerate}
    \item We extend the predictable-representation result for degenerate diffusions in \cite{ust_mart} to random initializations.
    \item In Section \ref{sect:tweedie}, we introduce local regression losses that do not require a density or score {function}.  Under local square integrability of the
{    conditional Radon--Nikodym derivative and a state-feedback factorization, the corresponding minimizer is the unique minimum-energy Girsanov shift.}
    
    \item We present this loss function as a solution of a control problem (Subsection \ref{sect:control}). Our loss functions depend on choosing sufficient increments of the data, say $\varepsilon$. To demonstrate the performance of our loss function under different $\varepsilon$, in Subsection \ref{sec:degenerate-diffusion} we provide a numerical experiment. We also provide an application to an image generation problem in Subsection \ref{sect:cond}.
\end{enumerate}

\section{Notation}

For a process $X$ {on a filtered probabiltiy space $(\Omega,\mathcal F,(\mathcal G_t)_t,\mathbb P)$}, we let
$\mathcal F_t^0(X):=\sigma\{X_s:0\leq s\leq t\}$ and let
$\mathcal F_t(X)$ denote its usual $\mathbb P$-augmentation.  For another
probability measure $\mathbb Q$, $\mathcal F_t^{\mathbb Q}(X)$ denotes the
corresponding $\mathbb Q$-augmentation.  The ambient filtration is denoted by
$(\mathcal G_t)_{t\leq T}$.
For a measure $\mathbb Q$ and a $\sigma$-field $\mathcal H$,
$\mathbb Q|_{\mathcal H}$ denotes restriction to $\mathcal H$.  We write
$\mathcal L_{\mathbb Q}(U)$ for the law of $U$ under $\mathbb Q$ and
$D_{\mathrm{KL}}(\mathbb Q\Vert\mathbb P)$ for relative entropy.  The space
$L_a^2(\mathcal F(X),\mathbb Q;\mathbb R^d)$ consists of
$\mathcal F(X)$-predictable processes $v$ such that
$\mathbb E^{\mathbb Q}\int_0^T|v_t|^2\,dt<\infty$.
The notation $L_a^2([0,r],\mathcal F(X),\mathbb Q;\mathbb R^d)$ uses
the same definition with the integral restricted to $[0,r]$, and
$L_{a,\mathrm{loc}}^2([0,T),\mathcal F(X),\mathbb Q;\mathbb R^d)$ means
that this condition holds on every interval $[0,r]$ with $r<T$.
The symbol $A^\dagger$ denotes the Moore--Penrose pseudoinverse and
$\|A\|_{\mathrm F}$ the Frobenius norm.  {Given a matrix valued curve $a: [0,T] \to \mathbb R^{n \times d},$} we use
\[
\widetilde \Pi(t,a):=\Pi(a(t)):=a(t)^\top
\bigl(a(t)a^\top(t)\bigr)^\dagger a(t),
\qquad P_t(X):=\widetilde \Pi(t,\sigma(\cdot,X_{\cdot}))=\Pi(\sigma(t,X_t)),
\]
i.e., $\Pi(a(t))$ is the orthogonal projection onto
{$\operatorname{Ran}(a(t)^{\top})=(\ker a(t))^\perp$ under Euclidean inner product; hence}
$P_t(X)$ is {an orthogonal projection} onto
$\operatorname{Ran}(\sigma(t,X_t)^\top)=(\ker\sigma(t,X_t))^\perp$.
All process equalities are understood up to the product measure explicitly
specified in the statement.

\section{Loss Functions for Degenerate Conditioned Diffusions}\label{sect:tweedie}

In this section, we introduce our main regression formulae and explain how
they can be used as population losses.  The predictable-representation and
minimum-energy arguments are given in Section~\ref{sect:6}, and the detailed
proofs are deferred to the appendix.

\begin{assumption}[Regularity and well-posedness]\label{ass:1}
On a complete filtered probability space
$(\Omega,\mathcal F,(\mathcal G_t)_{t\leq T},\mathbb P)$ satisfying the
usual conditions, $B$ is a $d$-dimensional $(\mathcal G_t)$-Brownian motion
and $X_0$ is $\mathcal G_0$-measurable.  \Eqref{eq:s} has a
pathwise-continuous strong solution that is unique up to
indistinguishability.  The maps
$b:[0,T]\times\mathbb R^n\to\mathbb R^n$ and
$\sigma:[0,T]\times\mathbb R^n\to\mathbb R^{n\times d}$ are bounded and
Borel measurable.  If $\mu_0=\mathcal L_{\mathbb P}(X_0)$, then for
$\mu_0$-a.e.\ $x$ the equation started from $x$ has a weak solution that is
unique in law.  A regular conditional family
$(\mathbb P_x)_{x\in\mathbb R^n}$ of $\mathbb P$ given $X_0$ exists.
\end{assumption}

Global Lipschitz continuity in the state variable together with linear growth
{is a standard sufficient condition for the existence of a unique strong solution for \eqref{eq:s}.  Boundedness is used below only for simple integrability estimates and can 
be replaced by corresponding localized moment assumptions. We only use pathwise uniqueness of \eqref{eq:s} for the minimum entropy property of the Girsanov shift $u$ that we use to construct \eqref{eq:s2}.}

\begin{assumption}[Regularity of conditional density]\label{ass:2}
Let $\mathbb P^y:=\mathbb P(\,\cdot\mid Y=y)$ be a regular conditional
probability.  For $\mathcal L_{\mathbb P}(Y)$-a.e.\ $y$ and every $t<T$,
\[
\mathbb P^y|_{\mathcal F_t(X)}\ll
\mathbb P|_{\mathcal F_t(X)},
\qquad
Z_t^y(X):=
\frac{d(\mathbb P^y|_{\mathcal F_t(X)})}
     {d(\mathbb P|_{\mathcal F_t(X)})}
\in L^2(\mathcal F_t(X),\mathbb P).
\]
\end{assumption}

When $y$ is fixed, we suppress the superscript and write $Z_t(X)$ for
$Z_t^y(X)$. The only restrictve condition that is used in this work is the square integrability of the process $Z^y(X)=(Z^y_t(X))_t$, which is needed to apply the martingale representation theorem.

\begin{remark}[Square integrability and entropy]\label{rem:l2-entropy}
The $L^2(\mathbb P)$ condition is used to invoke the predictable-representation theorem, which is the primary additional restriction in our work. Since for $x >0$ it holds that $x \log x \le C(1+x^2)$, Assumption~\ref{ass:2} also implies that we have finite entropy: $\mathbb E^{\mathbb P}[Z^y_t(X)\log Z^y_t(X)] <\infty$ for every $t <T$. 
\end{remark}

\begin{assumption}\label{ass:feedback}
There exists a Borel measurable function $G$ such that $Y=G(X_T)$ and solutions of \eqref{eq:s} are Markov processes.
\end{assumption}

\begin{remark}
    We need Assumpton \ref{ass:feedback} to have loss functions that rely on $(X_t,Y)$ at time $t$ rather than $(X_{[0,t]},Y),$ where $X_{[0,t]}$ denotes the whole trajectory of the process $X$ up to time $t$.
\end{remark}

For the rest of this section, Assumptions~\ref{ass:1}, \ref{ass:2}, and \ref{ass:feedback} are in force.

\begin{theorem}[Degenerate Tweedie formula-1]\label{thrm:quaso}
For $\mathcal L_{\mathbb P}(Y)$-a.e.\ $y$ and every $r<T$, there is a
unique realized process
$u_{\cdot}(X_\cdot,y)\in
L_{a,\mathrm{loc}}^2([0,T),\mathcal F(X),\mathbb P^y;\mathbb R^d)$
satisfying
\begin{equation}\label{eq:control}
u_t(X_t,y)
=
\lim_{\varepsilon\downarrow0}
\mathbb E^{\mathbb P^y}
\left[
\left.
\frac1\varepsilon
\int_t^{t+\varepsilon}P_s(X)\,dB_s
\right|
 X_t
\right]
=
\lim_{\varepsilon\downarrow0}
\mathbb E^{\mathbb P}
\left[
\left.
\frac1\varepsilon
\int_t^{t+\varepsilon}P_s(X)\,dB_s
\right|
 X_t,Y=y
\right]
\end{equation}
where both limits are in
$L^2([0,r]\times\Omega,dt\otimes d\mathbb P^y)$.  Moreover,
{$P_t(X)u_t(X_t,y)=u_t(X_t,y)$ $dt\otimes d\mathbb P^y$-a.e.  There is a corresponding weak solution of \eqref{eq:s2} whose state-path law on $[0,r]$}
is $\mathcal L_{\mathbb P}(X_{[0,r]}\mid Y=y)$, which is given by the solution of \eqref{eq:s} under the measure $Z^y_t(X) d\mathbb P$ with the driving Brownian motion $B_t - \int_0^t u_s(X_s,y)ds$.   Among the admissible
{Girsanov shifts $v$, $u$ uniquely minimizes
$\mathbb E^{\mathbb P^y}\int_0^r|v_t|^2\,dt$.}
\end{theorem}

\begin{remark}[Independent simulation]\label{rem:controlled-uniqueness}
{The proof relies on Girsanov theorem, which proves existence of a weak controlled}
solution with the desired law.  If the controlled \eqref{eq:s2} is
weakly unique, then every weak solution of that equation has the same law.
Weak uniqueness of the original equation alone does not imply weak
uniqueness after adding a merely measurable control.
\end{remark}

\begin{corollary}[Degenerate Tweedie formula-2]\label{cor:sd}
\Eqref{eq:control} can also be written as
\begin{equation}\label{eq:control2}
u_t(X_t,y)
= \lim_{\varepsilon\downarrow0} 
\mathbb E^{\mathbb P^y} \left[ \left. \frac{B_{t+\varepsilon}-B_t}\varepsilon \right| X_t \right]
= \lim_{\varepsilon\downarrow0} 
\mathbb E^{\mathbb P} \left[ \left. \frac{B_{t+\varepsilon}-B_t}\varepsilon \right| X_t , Y=y\right]. 
\end{equation}
\end{corollary}

{A Girsanov shift $u$ changes the state drift only through $\sigma u$.  A component}
of $u$ in $\ker\sigma$ would therefore leave the state law unchanged while
increasing its energy.  The minimum-energy control must consequently satisfy
{$P_t(X)u_t=u_t$ a.e. Corollary~\ref{cor:sd} is useful in simulation because the} Brownian increments are already available and no numerical projection is
required.

{If only the sample trajectories of the states $X$ and the condition $Y$ are observed, one can instead
learn the induced shift $\sigma u$ on the drift of \eqref{eq:s} instead of $u$.}

\begin{theorem}[Degenerate Tweedie formula-3]
\label{thm:state-tweedie}
Let $u_t(X_t,y)$ be the control obtained from \eqref{eq:control}.
Then, for every $r<T$,
\begin{equation}
\label{eq:state-tweedie}
\sigma(t,X_t)u_t(X_t,y)
=
\lim_{\varepsilon\downarrow0}
\mathbb E^{\mathbb P^y}
\left[
\left.
\frac{X_{t+\varepsilon}-X_t}{\varepsilon}
-
b(t,X_{t})
\right|
 X_t
\right]
\end{equation}
exists in $L^2\bigl([0,r]\times\Omega,dt\otimes d\mathbb P^y\bigr).$
\end{theorem}

Under the conditioned measure $\mathbb P^y$, enlarging the information by $Y$ changes the
semimartingale decomposition of $X$.  The predictable-representation result
below (Theorem \ref{thrm:rep}) identifies the new finite-variation term as $\sigma u$ and rules out an
additional orthogonal martingale contribution.  Although $\sigma u$ does not
identify an arbitrary Brownian control modulo $\ker\sigma$, it does identify
{the minimum-energy control. Indeed,}
\[
\sigma(t,X_t)^\dagger\sigma(t,X_t)u_t
=P_t(X)u_t=u_t,
\]
where
{$\sigma^\dagger=\sigma^\top(\sigma\sigma^\top)^\dagger$.  Hence, \eqref{eq:state-tweedie} recovers the control $u$ through the
Moore--Penrose inverse even when $\sigma$ is noninjective, which then can be used as a Girsanov shift as described in Theorem \ref{thrm:quaso}.}

The next corollary recovers the classical Tweedie formula when it is
available.  Notice that the preceding recovery of the  control
does not recover the entire Brownian path.  In a degenerate model, $B$ cannot
generally be reconstructed from
$\int_0^\cdot\sigma(s,X_s)\,dB_s$. The missing kernel component is irrelevant
to the minimum-energy control and to the controlled state \eqref{eq:s2}.

\begin{corollary}\label{cor:2}
{Suppose that $Z_t^y=h(t,X_t;y)$ with $h>0$, and that $h$ admits the It\^o differential}
\[
dh(t,X_t;y)
=\left\langle\sigma(t,X_t)^\top\nabla_xh(t,X_t;y),dB_t\right\rangle;
\]
for example, this holds for a $C^{1,2}$ solution of the corresponding
backward equation.  Then,
\[
u_t(X_t,y)=\sigma(t,X_t)^\top\nabla_x\log h(t,X_t;y)
\quad dt\otimes d\mathbb P^y\text{-a.e.}
\]
\end{corollary}

The desired control $u$ can be approximated by the following local regression problem.

\begin{proposition}\label{prop}
Fix $r<T$ and $0<\varepsilon<T-r$.
The problem
\begin{equation}\label{eq:amortized-projected-loss}
\inf_{v\in L^2_a(\mathcal F(X),dt\otimes d\mathbb P^y;\mathbb R^d)}
\mathbb E^{\mathbb P^y}
\int_0^r
\left|
v_t-
\frac1\varepsilon
\int_t^{t+\varepsilon}P_s(X)dB_s
\right|^2dt.
\end{equation}
has the unique minimizer given by the product-space conditional expectation that satisfies
\[
u_\varepsilon(t,X_t,y)
=\mathbb E^{\mathbb P^y}\left[
\left.\frac1\varepsilon\int_t^{t+\varepsilon}P_s(X)\,dB_s
\right|X_t\right]
\quad dt\otimes d\mathbb P^y\text{-a.e.}
\]
Moreover, $u_{\varepsilon}\to u_{\cdot}(X_\cdot,y)$
in $L^2([0,r]\times\Omega,dt\otimes d\mathbb P^y)$.
\end{proposition}
\begin{proof}
Conditional expectation is the orthogonal projection onto the closed
product-space subspace $L^2_a(\mathcal F(X),dt\otimes d\mathbb P^y;\mathbb R^d)$.  The convergence is from
Theorem~\ref{thrm:quaso}.
\end{proof}

\begin{remark}
Analogous projection statements hold for Corollary~\ref{cor:sd} and
Theorem~\ref{thm:state-tweedie}.  Replacing the targets $S_{t_j}^i$ in
Algorithm~\ref{alg:training-step} leads to the corresponding algorithms.
\end{remark}

\begin{algorithm}[t]
\caption{Training with \eqref{eq:amortized-projected-loss}}
\label{alg:training-step}
\begin{algorithmic}[1]
\Require SDE \eqref{eq:s}, batch size $N$, current control approximation
$u^\theta$, uniform grid $t_k=kh$, window $\varepsilon=qh$, and $r<T-\varepsilon$.

\For{$i=1$ to $N$}
    \State Sample a path $X^i$ and its corresponding Brownian path $B^i$
    from \eqref{eq:s}.
    \State Compute the corresponding condition $Y^i=G(X_T^i)$.
    \State For every $t_j\leq r$, compute the discrete projected-noise target
    \Statex
    \begin{equation}\label{eq:proj}
    S_{t_j}^i
    =
    \frac{1}{qh}
    \sum_{k=j}^{j+q-1}
    P_{t_k}(X^i)
    \bigl(B_{t_{k+1}}^i-B_{t_k}^i\bigr).
    \end{equation}
    \State Calculate the single-path loss
    \Statex
    \begin{equation}\label{eq:reg}
    \ell_i(\theta)
    =
    h\sum_{\{j:\,t_j\leq r\}}
    \left\|
    u_{t_j}^{\theta}(X_{t_j}^i,Y^i)-S_{t_j}^i
    \right\|^2.
    \end{equation}
\EndFor

\State Calculate the Monte Carlo full-batch loss
\Statex
\[
\mathcal L^N(\theta)
=
\frac{1}{N}\sum_{i=1}^{N}\ell_i(\theta).
\]
\State Take a gradient step on $\mathcal L^N(\theta)$.
\end{algorithmic}
\end{algorithm}

In Algorithm~\ref{alg:training-step} we adapt
\cite[Algorithm~1]{pidstrigach2025conditioning} to the degenerate setting under our loss functions.

For example, for an overparameterized ReLU parametrization with finite-dimensional separated inputs, vanishing finite-sample optimization error can be justified under the random-initialization, width, data-separation, and learning-rate conditions of \citet{allenzhu2019convergence}, which are sufficient to establish convergence of the algorithm considered in \citet{tzen2019theoretical}. In particular, \citet{allenzhu2019convergence} establish linear convergence of gradient descent and stochastic gradient descent for $\ell_2$ regression. When these specific conditions are not imposed, we retain vanishing optimization error as an explicit assumption. We discuss the corresponding implementation details in Appendix~\ref{app:training-consistency}. We emphasize, however, that the regression step \eqref{eq:reg} must in practice be approximated using a neural network.

\section{Applications}

\subsection{A Stochastic Control Problem}\label{sect:control}

The variational formula of \cite{boue1998variational} is interpreted in
\cite{tzen2019theoretical} as a stochastic control problem with a quadratic
running cost. This control problem approach has been used in the generative model for conditioning in recent works, for instance, see \cite{domingo2025adjoint}. We next extend this interpretation to the degenerate setting.
The statement is informal; Theorem~\ref{thrm:10} in Appendix \ref{app:d} gives the precise version
for an arbitrary target endpoint law.

\begin{theorem}[Informal]\label{thm:degenerate-stochastic-control}
Suppose that $X_0=x_0$ is deterministic, and let
\[
z_r^y :=
\frac{d\mathcal L_{\mathbb P}(X_r\mid Y=y)}
     {d\mathcal L_{\mathbb P}(X_r)}.
\]
{Define the endpoint density process and the control by}
\[
Z_t^{y,r}(X):=
\mathbb E^{\mathbb P}[z_r^y\mid\mathcal F_t(X)]
=1+\int_0^t\langle\xi_s^{y,r},dB_s\rangle,
\qquad
u_t^{y,r}:=\frac{\xi_t^{y,r}}{Z_t^{y,r}(X)}
\quad\text{on }\{Z_t^{y,r}(X)>0\}.
\]
Consider
\begin{equation}
\label{eq:control-problem-main}
V(r,y)
:=
\inf_{(\mathbb Q,v)}
\mathbb E^{\mathbb Q}
\left[
\frac12\int_0^r|v_t|^2dt
-
\log z_r^y
\right],
\end{equation}
{where the infimum is taken over admissible Girsanov shifts.}

Then, $V(r,y)=0$.  The endpoint control $u^{y,r}$ is the unique minimizer in
{a class of admissible one-sided Girsanov shifts, and the measure}
$\mathbb Q^{y,r}$ satisfies
\[
\mathcal L_{\mathbb Q^{y,r}}(X_r)
=
\mathcal L_{\mathbb P}(X_r\mid Y=y).
\]
{Among all admissible Girsanov shifts of \eqref{eq:s} that satisfies this conditional
terminal law, $u^{y,r}$ is the unique Girsanov shift with the minimum-entropy property:}
\[
\frac12
\mathbb E^{\mathbb Q^{y,r}}
\int_0^r|u_t^{y,r}|^2dt
=
\inf_{\substack{(\mathbb Q,v):\\
\mathcal L_{\mathbb Q}(X_r)
=
\mathcal L_{\mathbb P}(X_r\mid Y=y)}}
\frac12
\mathbb E^{\mathbb Q}
\int_0^r|v_t|^2dt.
\]
\end{theorem}

\subsection{Degenerate Nonlinear Diffusion}
  \label{sec:degenerate-diffusion}

  We consider the nonlinear diffusion
  \begin{equation}
      dX_t=b(X_t)\,dt+\sigma\,dB_t,
      \qquad
      X_0=(1,0,0.4)^\top,
      \qquad
      Y=X_T^{(1)},
  \end{equation}
  where $B_t$ is a standard three-dimensional Brownian motion and
  \begin{equation}
      b(x)=
      \begin{pmatrix}
          20x_1(1-x_1^2)\\
          -x_2+0.35x_1\\
          -0.8x_3
      \end{pmatrix},
      \qquad
      \sigma=
      \begin{pmatrix}
          1&1&1\\
          0&1&2\\
          0&2&4
      \end{pmatrix}.
  \end{equation}
  Since
  \(\operatorname{rank}(\sigma)=2<3\), the instantaneous covariance
  $
      \sigma\sigma^\top
      =
      \begin{pmatrix}
          3&3&6\\
          3&5&10\\
          6&10&20
      \end{pmatrix}
  $
  is singular. The corresponding projection to $\mathrm{Ran}(\sigma^{\top})$ is
  $
      P=\sigma^\dagger\sigma
      =
      \begin{pmatrix}
          5/6&1/3&-1/6\\
          1/3&1/3&1/3\\
          -1/6&1/3&5/6
      \end{pmatrix}.
  $

  For each \(\varepsilon\in\{0.005,0.01,0.1\}\), we train a separate control
  using the fixed-\(\varepsilon\) loss
  \begin{equation}
      \mathcal L_\varepsilon(\theta)
      =
      \mathbb E\left[
          \left\|
              P u_\theta(t,X_t,Y)
              -
              P\frac{B_{t+\varepsilon}-B_t}{\varepsilon}
          \right\|^2
      \right].
      \label{eq:degenerate-fixed-epsilon-loss}
  \end{equation}
  The training pool contains \(80{,}000\) independent reference
  trajectories. Each model is trained for \(20{,}000\) Adam updates with
  batch size \(2048\).

  At evaluation time, we impose the terminal condition \(Y=-1\) and
  simulate
  \begin{equation}
      d\widehat X_t
      =
      \left[
          b(\widehat X_t)
          +\sigma P u_\theta(t,\widehat X_t,-1)
      \right]dt
      +\sigma\,dB_t,
  \end{equation}
  where $Pu_{\theta}$ is obtained from \eqref{eq:degenerate-fixed-epsilon-loss} through a neural network \cite{tzen2019theoretical}.
  Because \(X_t^{(1)}\) is autonomous, the ground-truth can be
  computed from the one-dimensional backward Kolmogorov equation. Writing
  \begin{equation}
      h(t,x_1)
      =
      p\!\left(
          X_T^{(1)}=-1
          \,\middle|\,
          X_t^{(1)}=x_1
      \right),
  \end{equation}
  the corresponding backward Kolmogorov equation has diffusion term
{  \(\frac{3}{2}\partial_{x_1x_1}h\). The corresponding exact Girsanov shift is}
  \begin{equation}
      u^\star(t,x)
      =
      \sigma^\top\nabla_x\log h(t,x_1)
      =
      \begin{pmatrix}
          \partial_{x_1}\log h(t,x_1)\\
          \partial_{x_1}\log h(t,x_1)\\
          \partial_{x_1}\log h(t,x_1)
      \end{pmatrix},
  \end{equation}
{  and this induces the following shift on the drift term of \eqref{eq:s}}
  \begin{equation}
      \sigma u^\star(t,x)
      =
      (3,3,6)^{\top}
      \partial_{x_1}\log h(t,x_1),
  \end{equation}
  which we use to compare our empirical results against the ground-truth.
  We evaluate each method using \(15{,}000\) independently simulated
  trajectories. Sample trajectories obtained from this experiment can be
  seen on Figure \ref{fig:degenerate-brownian-trajectories}.

  \begin{figure}[t]
      \centering
      \includegraphics[width=\linewidth]{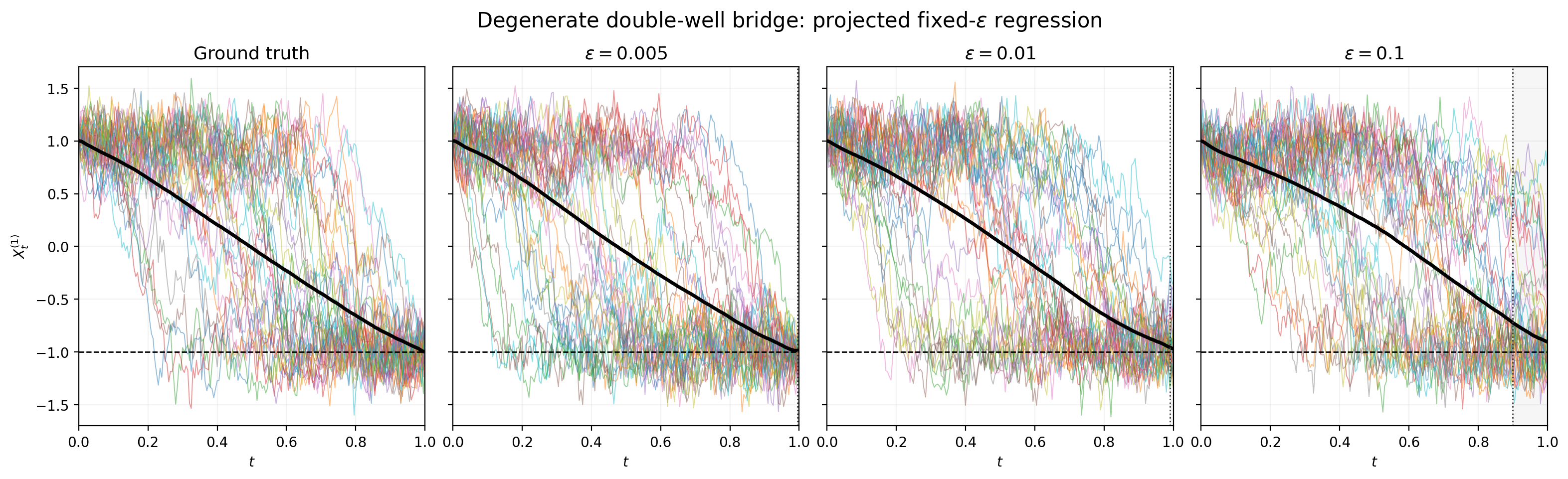}
      \caption{
          Sample paths from the numerical ground-truth Doob bridge and
          from the learned controls with
          \(\varepsilon=0.005\), \(\varepsilon=0.01\), and
          \(\varepsilon=0.1\).
          Thin curves represent individual trajectories, the bold black
          curve is the empirical mean, and the horizontal dashed line
          denotes the terminal target \(X_T^{(1)}=-1\).
      }
      \label{fig:degenerate-brownian-trajectories}
  \end{figure}

\subsection{Conditioned Image Generation}\label{sect:cond}

We next give a proof-of-concept application to image inpainting.  In contrast
to Algorithm~\ref{alg:training-step}, the terminal samples in this experiment
are drawn from an empirical data law rather than from the terminal law of the
reference SDE. Our implementation of this experiment closely follows the one in \cite{baker2024conditioning}.

Degenerate models are natural when the variability relevant to the
conditioning task is intrinsically lower-dimensional \cite{li2026back}. For example, in an
inpainting problem in which only a localized feature is missing, one can use
the accessible directions of the diffusion to represent the variability of
the missing region while conditioning on the observed context.

In our numerical illustration, we train on the Shirt class of
Fashion-MNIST \cite{xiao2017fashion}, condition on the left half of a $28\times28$ image, and
generate its missing right half. The forward model is
\[
dX_t=AS\,dB_t,
\qquad X_0=x_0,
\]
where the deterministic $x_0$ is the empirical training mean, $A$
embeds the accessible latent directions into the image space, and $S$
rescales these directions according to their empirical standard
deviations. The latent coordinates comprise the $392$ whitened visible
pixels and $256$ whitened principal components of the hidden half.  Before
training, each hidden half is projected onto these retained principal
components, so the reconstructed terminal target lies in the affine support
of the model.
Hence, the accessible dimension is $648<784$, and thus every generated path
remains on a fixed affine subspace and no image-space density
or score {function} is required.

Let $\nu_T$ be the empirical law of the projected terminal images.  We first
sample a target $E\sim\nu_T$ and then sample an additive Gaussian bridge from
$x_0$ to $E$.  This defines a bridge-mixture law $\mathbb R^{\nu_T}$; it is
not the reference law of $dX_t=AS\,dB_t$.  Although $\nu_T$ is atomic and
singular at $T$, for every observed condition $y$ and every $r<T$, the
conditional bridge-mixture marginal
$\nu_r^y=\mathcal L_{\mathbb R^{\nu_T}}(X_r\mid Y=y)$ is absolutely
continuous with respect to the reference marginal $\mu_r$. 

The regression expectation in this experiment is taken under
$\mathbb R^{\nu_T}$, not under the reference path law $\mathbb P$.  We use
the finite target
$
\frac{B_{t+\varepsilon}-B_t}{\varepsilon},.
$
This choice does not send $\varepsilon$ to zero.

Since the diffusion is supported on a $648$-dimensional affine subspace of
the $784$-dimensional image space, $X_t$ is singular with respect to the
{Lebesgue measure and there is no available score function}
$\nabla_x\log p_t(x)$.  Our construction requires neither such a score {function} nor a
smooth extension of the conditional likelihood.  Since the left half
of the terminal image is observed, its corresponding control is available
analytically.  A small U-Net therefore learns only the unknown hidden
component.  We train its parameters using AdamW.  Its inputs are the current
decoded image, the observed left half, the inpainting mask, and time, and its
image-shaped output is mapped onto the retained PCA coordinates of the hidden
half.

At generation time, an exponential moving average of the trained
network repeatedly predicts the hidden control. Each Euler--Maruyama
step combines the predicted control with a new Gaussian increment in
the accessible coordinates and decodes the updated coordinates back
into image space. The time grid is refined near the terminal time,
producing multiple right-half completions consistent with the supplied
left half.  Results are shown in Figure~\ref{fig:completion-grids}.

\begin{figure}
    \centering
    \includegraphics[width=1\linewidth]{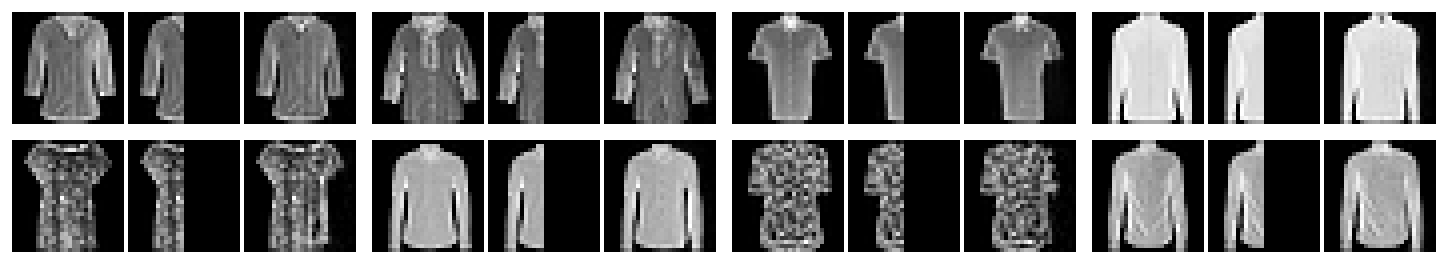}
    \caption{Conditioned image-generation results for eight held-out examples. First column is the ground truth, the second column is the given condition, and the last column is the generated image. The SDE we used for conditioned model is degenerate, cannot be ``time-reverted'' via the methods used in the existing literature and $\nabla_x \log p(Y=y|X_t=x)$, and $\sigma^{-1}$ do not exist.}
    \label{fig:completion-grids}
\end{figure}
\section{Predictable Representation Property and Causal Optimal Transport}\label{sect:6}

The loss functions we derived rely on the predictable representation property of diffusion processes when the corresponding martingale problem is well-posed. To derive those loss functions, we extend the martingale representation result of \cite{ust_mart} to a
random initialization in the conditioned setting.  Related constant-rank results for manifold-valued
diffusions appear in \cite{elworthy2010geometry}.  However, when the rank is not
constant, the associated geometric distributions need not retain the same
smoothness; see \cite[Theorem 10.34]{lee2013smooth} and
\cite[Subsection 9.2]{elworthy2010geometry}.

\begin{theorem}\label{thrm:rep}
Let $(X,B)$ be a weak solution of \eqref{eq:s} on a complete filtered
probability space and $b$, and $\sigma$ have linear growth, where $B$ is a $(\mathcal G_t)$-Brownian motion and
$X_0$ is $\mathcal G_0$-measurable.  Let
$\mu_0=\mathcal L_{\mathbb P}(X_0)$, and suppose that for
$\mu_0$-a.e.\ $x$ \eqref{eq:s} with $X_0=x$ has a weak solution
that is unique in law. 
Then, every $F\in L^2(\mathcal F_T(X),\mathbb P;\mathbb R)$ has a
representation
\begin{equation}\label{eq:rep}
F=\mathbb E^{\mathbb P}[F\mid X_0]
+\int_0^T\langle H_s,dB_s\rangle,
\end{equation}
for a unique
$H\in L_a^2(\mathcal F(X),\mathbb P;\mathbb R^d)$ satisfying
$P_t(X)H_t=H_t$ $dt\otimes d\mathbb P$-a.e.  Conversely, if
$g(X_0)\in L^2(\sigma(X_0),\mathbb P)$ and $H$ has these properties, then
$g(X_0)+\int_0^T\langle H_s,dB_s\rangle$ belongs to
$L^2(\mathcal F_T(X),\mathbb P;\mathbb R)$.
\end{theorem}
\begin{proof}
    See Appendix \ref{app:c} for a complete proof.
\end{proof}

{In particular, predictability and square integrability of the integrand is not sufficient to ensure that $F$ is $\mathcal F_T(X)$ measurable. In our next result, we essentially show that the corresponding minimum entropy Girsanov shifts therefore need to be of the form $\sigma u$ as in \eqref{eq:s2}, which generalizes \cite[Theorem 3]{ust_mart} when the diffusion has random initialization.}

\begin{theorem}\label{thrm}
Under the hypotheses of Theorem~\ref{thrm:rep}, let
$\mathcal H_t:=\sigma(X_0)\vee\mathcal F_t(B)$, with the usual
augmentation, and suppose $F\in L^2(\mathcal H_T,\mathbb P;\mathbb R)$.
Write
\begin{equation}\label{bm:rep}
F=\mathbb E^{\mathbb P}[F\mid X_0]
+\int_0^T\langle\xi_s,dB_s\rangle,
\end{equation}
where
$\xi\in L_a^2((\mathcal H_t)_{t\leq T},\mathbb P;\mathbb R^d)$.
Then,
\[
\mathbb E^{\mathbb P}[F\mid\mathcal F_T(X)]
=\mathbb E^{\mathbb P}[F\mid X_0]
+\int_0^T\langle P_s(X)\mathbb E^{\mathbb P}[\xi_s|\mathcal F_s(X)],dB_s\rangle.
\]
\end{theorem}
\begin{proof}
    See Appendix~\ref{app:c} for a complete proof.
\end{proof}

The representation in \eqref{bm:rep} is satisfied for any such $F$. Using this representation, we derive a nonsmooth analogue of the Doob $h$-transform and connect it to causal optimal transport
{\cite{lassalle2012invertibility,lassalle2018causal}. In particular, using Theorem~\ref{thrm}, we present a method to approximate the corresponding minimum entropy shift over all Girsanov shifts. As shown in Corollary~\ref{cor:2}, this naturally generalizes the Doob $h$-transform while preserving the desired properties.}

\begin{theorem}\label{thm:main}
Suppose that Assumptions~\ref{ass:1}--\ref{ass:2} hold.  Let the locally
square-integrable 
integrand $\xi$ be defined by
\[
Z_t(X)=Z_0(X)+\int_0^t\langle\xi_s,dB_s\rangle,
\qquad P_s(X)\xi_s=\xi_s,
\]
which exists and is unique by Assumption~\ref{ass:2} and Theorem~\ref{thrm:rep}. 
Set $u_t=\xi_t/Z_t(X)$ on $\{Z_t(X)>0\}$ and define it arbitrarily on
$\{Z_t(X)=0\}$.  Then, for every $r<T$,
\begin{equation}\label{limit}
\lim_{\varepsilon\downarrow0}
\mathbb E^{\mathbb P^y}\left[
\left.\frac1\varepsilon\int_t^{t+\varepsilon}P_s(X)\,dB_s
\right|X_t\right]
=u_t(X_t,y)
\end{equation}
in $L^2([0,r]\times\Omega,dt\otimes d\mathbb P^y)$.
\end{theorem}

\begin{proof}
    See Appendix~\ref{app:c} for a complete proof.
\end{proof}

{Our next result verifies the minimum-energy property of $u$ without requiring $Z$}
to be strictly positive under $dt\otimes d\mathbb P$, which extends \cite[Theorem 7]{ust_mart} to our setting. In particular, we verify that the limit \eqref{limit} has the minimum entropy property, as desired.

\begin{theorem}\label{thrm:causal}
Suppose that Assumptions~\ref{ass:1}--\ref{ass:2} hold, fix $r<T$, and let
$\xi$ be the local  integrand from
Lemma~\ref{lem:density-martingale}.  Set
$u_t=\xi_t/Z_t(X)$ on $\{Z_t(X)>0\}$ and $u_t=0$ otherwise.  Define
$H_0^y:=D_{\mathrm{KL}}\bigl(
\mathbb P^y|_{\mathcal F_0(X)}\Vert
\mathbb P|_{\mathcal F_0(X)}\bigr)$.  Let $\mathcal A_r^y$ be the class
of pairs $(\mathbb Q,v)$ such that $\mathbb Q\ll\mathbb P$ on
$\mathcal G_r$,
\[
\frac{d(\mathbb Q|_{\mathcal G_0})}
     {d(\mathbb P|_{\mathcal G_0})}=Z_0(X),
\]
$v$ is $(\mathcal G_t)$-predictable,
$B_t^{\mathbb Q}:=B_t-\int_0^t v_s\,ds$ is a
$(\mathcal G_t,\mathbb Q)$-Brownian motion,
$\mathbb E^{\mathbb Q}\int_0^r|v_t|^2dt<\infty$, and
\[
D_{\mathrm{KL}}(\mathbb Q\Vert\mathbb P)
=H_0^y+\frac12\mathbb E^{\mathbb Q}\int_0^r|v_t|^2dt.
\]
{For every $(\mathbb Q,v)\in\mathcal A_r^y$ such that $v$ is $\mathcal F(X)$ adapted and the equation}
\begin{equation}\label{eq:s:v}
dX_t
=
\bigl(b(t,X_t)+\sigma(t,X_t)v_t\bigr)dt
+
\sigma(t,X_t)dB_t^{\mathbb Q}
\end{equation}
{has the law $\mathcal L_{\mathbb P}(X_{[0,r]}|Y=y)$},
\begin{equation}\label{eq:comp}
\mathbb E^{\mathbb Q}\int_0^r |v_t|^2dt
\geq
\mathbb E^{\mathbb P^y}\int_0^r
|u_t|^2dt.
\end{equation}
{The pair $(\mathbb Q^{y,r},u)$, defined by}
$d(\mathbb Q^{y,r}|_{\mathcal G_r})/
d(\mathbb P|_{\mathcal G_r})=Z_r(X)$, belongs to
$\mathcal A_r^y$ and attains equality.  Equality holds for \eqref{eq:comp} if,
and only if, $\mathbb Q=\mathbb Q^{y,r}$ and
{$(v_t)_t=(u_t)_t$ $dt\otimes d\mathbb Q^{y,r}$-a.e.}
\end{theorem}

\begin{proof}
    See Appendix~\ref{app:c} for a complete proof.
\end{proof}

\begin{remark}
{The preceding theorem compares Girsanov shifts on a fixed probability space and therefore}
does not require uniqueness of the controlled equation.  Weak uniqueness of
{\eqref{eq:s2} is needed only to transfer the law $\mathcal L_{\mathbb Q^{y,r}}(X)$ to every
independently constructed controlled solution through Girsanov's theorem; see}
Remark~\ref{rem:controlled-uniqueness}.  Assumption~\ref{ass:feedback} is
{not needed to compare the relative entropy of these transformations; it is imposed in the regression
results only to express the Girsanov shift as a function of $(t,X_t,y)$ under the probability measure $\mathbb P^y$.}
\end{remark}

\section{Conclusion}

In this work, we have introduced regression losses for computing minimum-entropy
adapted shifts when the diffusion is degenerate and the conditional law may
be nonsmooth or singular.  The construction uses the predictable representation property of the diffusion and does not require a score function or an invertible diffusion coefficient. Under mild conditions, we have demonstrated that the obtained controls satisfy the minimum-entropy in a class of ``admissible controls''.

Extensions to manifold-valued and infinite-dimensional
diffusions are natural directions for future work.

\section{Impact Statement}

This work has connections to diffusion models, so it shares the societal and ethical implications of generative modelling.

\bibliographystyle{iclr2026_conference}
\bibliography{references}

@article {ust_mart,
    AUTHOR = {\"Ust\"unel, A. S.},
     TITLE = {Martingale representation for degenerate diffusions},
   JOURNAL = {J. Funct. Anal.},
  FJOURNAL = {Journal of Functional Analysis},
    VOLUME = {276},
      YEAR = {2019},
    NUMBER = {11},
     PAGES = {3468--3483},
}

@article{lassalle2018causal,
  title={Causal transport plans and their {M}onge--{K}antorovich problems},
  author={Lassalle, R{\'e}mi},
  journal={Stochastic Analysis and Applications},
  volume={36},
  number={3},
  pages={452--484},
  year={2018},
  publisher={Taylor \& Francis}
}

@article{lassalle2012invertibility,
  title={Invertibility of adapted perturbations of the identity on abstract {W}iener space},
  author={Lassalle, R{\'e}mi},
  journal={Journal of Functional Analysis},
  volume={262},
  number={6},
  pages={2734--2776},
  year={2012},
  publisher={Elsevier}
}

@article{bolhuis2002transition,
  title={Transition path sampling: Throwing ropes over rough mountain passes, in the dark},
  author={Bolhuis, Peter G and Chandler, David and Dellago, Christoph and Geissler, Phillip L},
  journal={Annual {R}eview of {P}hysical {C}hemistry},
  volume={53},
  number={1},
  pages={291--318},
  year={2002},
  publisher={Annual Reviews 4139 El Camino Way, PO Box 10139, Palo Alto, CA 94303-0139, USA}
}

@article{wang2011quantifying,
  title={Quantifying the {W}addington landscape and biological paths for development and differentiation},
  author={Wang, Jin and Zhang, Kun and Xu, Li and Wang, Erkang},
  journal={Proceedings of the National Academy of Sciences},
  volume={108},
  number={20},
  pages={8257--8262},
  year={2011},
  publisher={National Academy of Sciences}
}

@article{elerian2001likelihood,
  title={Likelihood inference for discretely observed nonlinear diffusions},
  author={Elerian, Ola and Chib, Siddhartha and Shephard, Neil},
  journal={Econometrica},
  volume={69},
  number={4},
  pages={959--993},
  year={2001},
  publisher={Wiley Online Library}
}

@inproceedings{zhang2023adding,
  title={Adding conditional control to text-to-image diffusion models},
  author={Zhang, Lvmin and Rao, Anyi and Agrawala, Maneesh},
  booktitle={2023 IEEE/CVF International Conference on Computer Vision (ICCV)},
  pages={3813--3824},
  year={2023},
  organization={IEEE}
}

@inproceedings{tzen2019theoretical,
  title={Theoretical guarantees for sampling and inference in generative models with latent diffusions},
  author={Tzen, Belinda and Raginsky, Maxim},
  booktitle={Conference on Learning Theory},
  pages={3084--3114},
  year={2019},
  organization={PMLR}
}

@inproceedings{
pidstrigach2025conditioning,
title={Conditioning Diffusions Using {M}alliavin Calculus},
author={Jakiw Pidstrigach and Elizabeth Louise Baker and Carles Domingo-Enrich and George Deligiannidis and Nikolas N{\"u}sken},
booktitle={Forty-second International Conference on Machine Learning},
year={2025},
}

@article{mirafzali2025malliavin,
  title={Malliavin calculus for score-based diffusion models},
  author={Mirafzali, Ehsan and Gupta, Utkarsh and Wyrod, Patrick and Proske, Frank and Venturi, Daniele and Marinescu, Razvan},
  journal={arXiv preprint arXiv:2503.16917},
  year={2025}
}

@inproceedings{blessing2503underdamped,
 author = {Blessing, Denis and Berner, Julius and Richter, Lorenz and Neumann, Gerhard},
 booktitle = {International Conference on Learning Representations},
 editor = {Y. Yue and A. Garg and N. Peng and F. Sha and R. Yu},
 pages = {2970--3002},
 title = {Underdamped Diffusion Bridges with Applications to Sampling},
 volume = {2025},
 year = {2025}
}

@article{conforti2025kl,
  title={{KL} convergence guarantees for score diffusion models under minimal data assumptions},
  author={Conforti, Giovanni and Durmus, Alain and Silveri, Marta Gentiloni},
  journal={SIAM Journal on Mathematics of Data Science},
  volume={7},
  number={1},
  pages={86--109},
  year={2025},
  publisher={SIAM}
}

@article {fabrice,
    AUTHOR = {Baudoin, Fabrice},
     TITLE = {Conditioned stochastic differential equations: theory,
              examples and application to finance},
   JOURNAL = {Stochastic Process. Appl.},
  FJOURNAL = {Stochastic Processes and their Applications},
    VOLUME = {100},
      YEAR = {2002},
     PAGES = {109--145},
}

@article{boue1998variational,
  title={A variational representation for certain functionals of {B}rownian motion},
  author={Bou{\'e}, Michelle and Dupuis, Paul},
  journal={The Annals of Probability},
  volume={26},
  number={4},
  pages={1641--1659},
  year={1998},
  publisher={Institute of Mathematical Statistics}
}

@article{dockhorn2021score,
  title={Score-based generative modeling with critically-damped {L}angevin diffusion},
  author={Dockhorn, Tim and Vahdat, Arash and Kreis, Karsten},
  journal={arXiv preprint arXiv:2112.07068},
  year={2021}
}

@article{debortoli2022convergence,
title={Convergence of denoising diffusion models under the manifold hypothesis},
author={De Bortoli, Valentin },
journal={Transactions on Machine Learning Research},
issn={2835-8856},
year={2022},
}

@inproceedings{de2022riemannian,
  title={Riemannian score-based generative modelling},
  author={De Bortoli, Valentin and Mathieu, Emile and Hutchinson, Michael and Thornton, James and Teh, Yee Whye and Doucet, Arnaud},
  booktitle={Advances in {N}eural {I}nformation {P}rocessing {S}ystems},
  volume={35},
  pages={2406--2422},
  year={2022}
}

@article{stroock1980extremal,
  title={On extremal solutions of martingale problems},
  author={Stroock, Daniel W and Yor, Marc},
  journal={Annales {S}cientifiques de l'{\'E}cole Normale Sup{\'e}rieure},
  volume={13},
  number={1},
  pages={95--164},
  year={1980}
}

@book{elworthy2010geometry,
  title={The {G}eometry of {F}iltering},
  author={Elworthy, K David and Le Jan, Yves and Li, Xue-Mei},
  year={2010},
  publisher={Springer Science \& Business Media}
}

@InProceedings{allenzhu2019convergence,
  title     = {A Convergence Theory for Deep Learning via Over-Parameterization},
  author    = {Allen-Zhu, Zeyuan and Li, Yuanzhi and Song, Zhao},
  booktitle = {Proceedings of the 36th International Conference on Machine Learning},
  pages     = {242--252},
  year      = {2019},
  volume    = {97},
  series    = {Proceedings of Machine Learning Research}
}

@article{xiao2017fashion,
  title={Fashion-{MNIST}: {A} novel image dataset for benchmarking machine learning algorithms},
  author={Xiao, Han and Rasul, Kashif and Vollgraf, Roland},
  journal={arXiv preprint arXiv:1708.07747},
  year={2017}
}

@book{lee2013smooth,
  title={Smooth {M}anifolds},
  author={Lee, John M},
  booktitle={Introduction to Smooth Manifolds},
  pages={1--31},
  year={2013},
  publisher={Springer}
}

@book{ikeda2014stochastic,
  title={Stochastic Differential Equations and Diffusion Processes},
  author={Ikeda, Nobuyuki and Watanabe, Shinzo},
  volume={24},
  year={2014},
  publisher={Elsevier}
}

@article{amendinger2000martingale,
  title={Martingale representation theorems for initially enlarged filtrations},
  author={Amendinger, J{\"u}rgen},
  journal={Stochastic Processes and Their Applications},
  volume={89},
  number={1},
  pages={101--116},
  year={2000},
  publisher={Elsevier}
}

@article{EJP431,
	author = {Thomas Kurtz},
	title = {The {Y}amada-{W}atanabe-{E}ngelbert theorem for general stochastic equations and inequalities},
	journal = {Electron. J. Probab.},
	fjournal = {Electronic Journal of Probability},
	volume = {12},
	year = {2007},
}

@book{revuz2013continuous,
  title={Continuous Martingales and Brownian Motion},
  author={Revuz, Daniel and Yor, Marc},
  year={2013},
  publisher={Springer Science \& Business Media}
}

@book {horn-MA,
    AUTHOR = {Horn, Roger A. and Johnson, Charles R.},
     TITLE = {Matrix {A}nalysis},
   EDITION = {Second},
 PUBLISHER = {Cambridge University Press, Cambridge},
      YEAR = {2013},
     PAGES = {xviii+643},
}

@inproceedings{baker2024conditioning,
  title={Conditioning non-linear and infinite-dimensional diffusion processes},
  author={Baker, Elizabeth L and Yang, Gefan and Severinsen, Michael L and Hipsley, Christy A and Sommer, Stefan},
  booktitle={Advances in Neural Information Processing Systems},
  volume={37},
  pages={10801--10826},
  year={2024}
}

@book{kuo2006gaussian,
  title={Gaussian {M}easures in {B}anach {S}paces},
  author={Kuo, Hui-Hsiung},
  pages={1--109},
  year={2006},
  publisher={Springer}
}

@book {MR1214374,
    AUTHOR = {Kloeden, Peter E. and Platen, Eckhard},
     TITLE = {Numerical {S}olution of {S}tochastic {D}ifferential {E}quations},
    SERIES = {Applications of Mathematics (New York)},
    VOLUME = {23},
 PUBLISHER = {Springer-Verlag, Berlin},
      YEAR = {1992},
     PAGES = {xxxvi+632},
}

@inproceedings{li2026back,
  title={Back to basics: Let denoising generative models denoise},
  author={Li, Tianhong and He, Kaiming},
  booktitle={Proceedings of the IEEE/CVF Conference on Computer Vision and Pattern Recognition},
  pages={36115--36125},
  year={2026}
}

@inproceedings{domingo2025adjoint,
  title={Adjoint matching: Fine-tuning flow and diffusion generative models with memoryless stochastic optimal control},
  author={Domingo i Enrich, Carles and Drozdzal, Michal and Karrer, Brian and Chen, Ricky TQ},
  booktitle={International Conference on Learning Representations},
  volume={2025},
  pages={53791--53846},
  year={2025}
}

@article {MR2280298,
    AUTHOR = {Nikeghbali, Ashkan},
     TITLE = {An essay on the general theory of stochastic processes},
   JOURNAL = {Probab. Surv.},
  FJOURNAL = {Probability Surveys},
    VOLUME = {3},
      YEAR = {2006},
     PAGES = {345--412},
}

@article{hairer2016advanced,
  title={Advanced stochastic analysis},
  author={Hairer, Martin},
  journal={Lecture Notes, https://www. hairer. org},
  year={2026}
}

\appendix

\section{A Primer on Stochastic Analysis}

We recall the notions of weak and strong well-posedness used in the paper.

\begin{definition}[Brownian motion]
Let $(\Omega,\mathcal F,(\mathcal G_t)_t,\mathbb P)$ be a complete
filtered probability space.  A $d$-dimensional Brownian motion is a
$(\mathcal G_t)$-adapted process $B$ with continuous paths such that

\begin{enumerate}
    \item $B_0=0$ $\mathbb P$-a.s.;
    \item for every $0\leq s<t$, the increment $B_t-B_s$ is independent of $\mathcal G_s$;
    \item for every $0\leq s<t$,
    $
    B_t-B_s\sim\mathcal N(0,(t-s)I_d).
    $
\end{enumerate}
\end{definition}

\begin{definition}[Weak solution]
A weak solution of \eqref{eq:s} consists of a filtered probability space
$(\Omega,\mathcal F,(\mathcal G_t)_t,\mathbb P)$, a
$(\mathcal G_t)$-Brownian motion $B$, and a continuous adapted process $X$
with the prescribed initial law such that
\[
X_t=X_0+\int_0^t b(s,X_s)\,ds
       +\int_0^t\sigma(s,X_s)\,dB_s
\quad dt\otimes \mathbb P\text{-a.s.}
\]
\end{definition}

\begin{remark}
In a weak solution, $X$ need not be adapted to the filtration generated by
$B$ alone; the filtration {generated by $B$} may contain additional randomness.
\end{remark}

\begin{definition}[Weak uniqueness]
\Eqref{eq:s} has \emph{weak uniqueness}, or uniqueness in law, if
any two weak solutions with the same initial law induce the same law for the
state process.  When needed, joint uniqueness refers to the law of $(X,B)$.
\end{definition}

\begin{definition}[Strong solution]
A solution is \emph{strong} if $X$ is adapted to the usual augmentation of
$\sigma(X_0)\vee\mathcal F_t(B)$; equivalently, it is constructed from the
given initial variable and driving Brownian motion without additional
randomness.
\end{definition}

\begin{definition}[Pathwise uniqueness]
\Eqref{eq:s} has \emph{pathwise uniqueness} if any two solutions on
the same filtered probability space, driven by the same Brownian motion and
with the same initial variable, are indistinguishable.
\end{definition}

\begin{theorem}[Yamada-Watanabe]
A weak solution exists and pathwise uniqueness holds if, and only if, a strong
solution exists and uniqueness in law holds.
\end{theorem}

The classical statement is for deterministic initial conditions; extensions
to more general stochastic equations are discussed in \cite{EJP431} and the
references therein.

In general, weak uniqueness need not imply pathwise uniqueness or strong existence.  A
standard example is Tanaka's equation
\[
dX_t = \mathrm{sign}(X_t)dB_t.
\]
Another example is Tsirelson's equation, which instead has a nonanticipative path-dependent drift,
and has the following form
\[
dX_t=b(t,X_{[0,t]})\,dt+dB_t.
\]
Our martingale representation theorem requires weak existence and uniqueness in law for almost every deterministic initial state $x$ under the conditioned measure $\mathbb P(\cdot | X_0=x)$; thus, it is applicable for the SDEs above although they do not admit a strong solution.

\begin{definition}
    Let $(\mathcal F_t^0)_{t\geq 0}$ be a filtration and let
$\mathcal N^{\mathbb P}$ denote the collection of all subsets of
$\mathbb P$-null sets in $\mathcal F_\infty^0$. The
\emph{$\mathbb P$-usual augmentation} of $(\mathcal F_t^0)_{t\geq 0}$ is the
filtration $(\mathcal F_t^{\mathbb P})_{t\geq 0}$ defined by
\[
\mathcal F_t^{\mathbb P}
:=
\bigcap_{s>t}
\left(
\mathcal F_s^0\vee\mathcal N^{\mathbb P}
\right),
\qquad t\geq 0.
\]
It is the smallest right-continuous and $\mathbb P$-complete
filtration containing $(\mathcal F_t^0)_{t\geq 0}$.
\end{definition}

{\begin{definition}[Girsanov shift]
Let $(\Omega,\mathcal F,(\mathcal F_t)_{t\in[0,T]},\mathbb P)$ support a
$d$-dimensional Brownian motion $B$, and let $\mathbb Q\ll\mathbb P$ on
$\mathcal F_T$ with density process
\[
Z_t
:=
\mathbb E^{\mathbb P}
\left[
\left.
\frac{d\mathbb Q}{d\mathbb P}
\right|
\mathcal F_t
\right].
\]
An $(\mathcal F_t)$-progressively measurable process
$u=(u_t)_{t\in[0,T]}$ is called a \emph{Girsanov shift} of $\mathbb Q$
relative to $\mathbb P$ if
\[
\int_0^T |u_t|^2\,dt<\infty
\qquad
\mathbb Q\text{-a.s.},
\text{ and }
B_t^{\mathbb Q}
:=
B_t-\int_0^t u_s\,ds,
\qquad 0\leq t\leq T,
\]
is a Brownian motion under $\mathbb Q$.
\end{definition}
}
\section{A Discussion of the Tweedie Formulae}

We first recall the Kunita--Watanabe decomposition, which explains the
potential orthogonal martingale component when the diffusion is degenerate.

\begin{theorem}[Kunita--Watanabe decomposition] 
Let \((\Omega,\mathcal F,(\mathcal F_t)_{t\ge0},\mathbb P)\) satisfy the
usual conditions.  Let \(M=(M^1,\dots,M^d)\) be a continuous
square-integrable martingale and let \(N\) be a square-integrable martingale.
Then, there exist a predictable \(M\)-integrable process \(H\) and a
square-integrable martingale \(L\) such that
\[N_t=N_0+\int_0^t\langle H_s,dM_s\rangle+L_t,\]
with
\[\langle L,M^i\rangle_t=0,
\qquad i=1,\dots,d,\quad t\ge0.\]

The decomposition is unique up to the usual
$d\langle M\rangle\otimes d\mathbb P$ equivalence for $H$ and
indistinguishability for $L$, where $\langle A,B\rangle_t$ denotes the quadratic variation between $A$ and $B$.
\end{theorem}

When $(\mathcal F_t)_t$ is a Brownian filtration, the Brownian martingale-representation theorem of It\^o
forces $L\equiv0$ when $M=B$.
{For a degenerate equation, however, {the filtration $(\mathcal F_t(X))_t$} generated by $X$ can be strictly smaller}
than the filtration generated by the driving Brownian motion $B$.  It is therefore not obvious that the
{orthogonal term $L$ vanishes after conditioning.  The martingale representation result}
for degenerate diffusions shows that the relevant Brownian integrand is unique; its
image $\sigma(t,X_t)u_t$ is the added state drift in \eqref{eq:s2}, where $(u_t)_t$ is $(\mathcal F_t(X))_t$-adapted. Thus, adaptiveness of $(u_t)_t$ is not sufficient for a stochastic integral of the form  $\int_0^{\cdot} \langle u_s,dB_s\rangle$ to be $(\mathcal F_t(X))_t$-adapted.

\section{Proofs of the Theorems in Section \ref{sect:6}}\label{app:c}

We first show that we can choose the projection $(P_t(X))_t$ as a Borel function of the coefficient.
{This choice is independent of any completion of the filtration generated by the state-process $X$ that solves \eqref{eq:s} and can}
therefore be used under both $\mathbb P$ and the conditional measures
$\mathbb P_x=\mathbb P(\,\cdot\mid X_0=x)$. This distinction is important, because in the original results of \cite{ust_mart}, the process $(P_t(X))_t$ is chosen via a measurable selection theorem; thus, it is not implementable in practice.

\subsection{Auxiliary Results}

\begin{lemma}\label{lem:projection-measurable}
{Let $(\mathcal G_t)_t$ be a filtration. Let $K=(K_t)$ be a $(\mathcal G_t)_t$-predictable process. Let \(z(\cdot,\cdot):[0,T]\times\mathbb R^n\to\mathbb R^{n\times d}\) be Borel measurable.  Define}
\[
\mathfrak q_t(\cdot)
:=
z(t,K_t(\cdot))^\top(z(t,K_t(\cdot))z(t,K_t(\cdot))^\top)^\dagger z(t,K_t(\cdot)).
\]
Then, \(\mathfrak q=(\mathfrak q_t)_t\) is $(\mathcal G_t)$-predictable. Moreover, \(\mathfrak q_t(\cdot)\)
is the orthogonal projection onto
$\operatorname{Ran}(z(t,K_t(\cdot))^\top)=(\ker z(t,K_t(\cdot)))^\perp.$
\end{lemma}

\begin{proof}
For \(a\in\mathbb R^{n\times d}\), {recall that}
$\Pi(a):=a^\top(aa^\top)^\dagger a$, where $(aa^\top)^\dagger$ is the
Moore--Penrose inverse of $aa^{\top}$ \cite[7.3P7]{horn-MA}.
By \cite[7.3P15]{horn-MA},
\[
\Pi(a)
=
\lim_{k\to\infty}
a^\top\bigl(aa^\top+k^{-1}I_n\bigr)^{-1}a,
\]
and for each $k$, the map on the right-hand side is continuous in $a$.
Consequently, $\Pi$ is Borel measurable.  Thus,
$\mathfrak q=(\Pi(z(t,K_t)))_t$ is predictable \cite[Chapter IX, Proposition 1.1]{revuz2013continuous}.  The projection property follows
directly from the Moore--Penrose pseudoinverse identities \cite[7.3P7]{horn-MA}.
\end{proof}

\begin{remark}
Since $(t,\omega)\mapsto(t,X_t(\omega))$ is predictable and $\sigma$ is
{Borel, $(\sigma(t,X_t))_t$ is predictable.  The projection process $P_{\cdot}(X)$ need not be}
continuous when the rank of $\sigma$ is not constant, but predictability is sufficient to have well-defined stochastic integrals. We also remark that, we did not assume that $(\mathcal G_t)_t$ is augmented.
\end{remark}

Second, for completeness, we show that $(Z_t(X))_t$ is a martingale.

\begin{lemma}\label{lem:density-martingale}
Suppose that Assumptions~\ref{ass:1}--\ref{ass:2} hold.  For
$\mathcal L_{\mathbb P}(Y)$-a.e.\ $y$, the process
\[
Z_t(X)
:=
\frac{d(\mathbb P^y|_{\mathcal F_t(X)})}
     {d(\mathbb P|_{\mathcal F_t(X)})},
\qquad 0\leq t<T,
\]
is a nonnegative $(\mathcal F_t(X),\mathbb P)$-martingale.  Moreover,
for every $R<T$, it admits a continuous version of the form
\[
Z_t(X)
=
Z_0(X)+\int_0^t\langle \xi_s,dB_s\rangle,
\qquad 0\leq t\leq R,
\]
where $P_s(X)\xi_s=\xi_s$ $dt\otimes d\mathbb P$-a.e., and
\[
Z_0(X) = \frac{d(\mathbb P^y|_{\mathcal F_0(X)})}
     {d(\mathbb P|_{\mathcal F_0(X)})}.
\]
\end{lemma}

\begin{proof}
For $0\leq s\leq t<T$ and $A\in\mathcal F_s(X)$, using the tower property and Assumption~\ref{ass:2}, we obtain
\[
\mathbb E^{\mathbb P}
\left[Z_t(X)\mathbf 1_A\right]
=
\mathbb P^y(A)
=
\mathbb E^{\mathbb P}
\left[Z_s(X)\mathbf 1_A\right].
\]
Hence
\[
\mathbb E^{\mathbb P}
\left[Z_t(X)\mid\mathcal F_s(X)\right]
=
Z_s(X) \qquad \mathbb P\text{-a.s.},
\]
and thus $Z$ is a nonnegative $(\mathcal F_t(X),\mathbb P)$-martingale.

Fix $R<T$. By Assumption \ref{ass:2},
$Z_R(X)\in L^2(\mathcal F_R(X))$. Therefore, Theorem
\ref{thrm:rep} gives an $\mathcal F(X)$-predictable process
$\xi^R\in L_a^2(\mathcal F(X),\mathbb P;\mathbb R^d)$ such that
\[
Z_R(X)
=
\mathbb E^{\mathbb P}[Z_R(X)\mid X_0]
+
\int_0^R\langle\xi_s^R,dB_s\rangle.
\]
By the martingale property,
\[
\mathbb E^{\mathbb P}[Z_R(X)\mid X_0]
=
Z_0(X).
\]
The  stochastic integral is an $\mathcal F(X)$-martingale.
Taking conditional expectation with respect to $\mathcal F_t(X)$ therefore
yields
\[
Z_t(X)
=
Z_0(X)+\int_0^t\langle\xi_s^R,dB_s\rangle,
\qquad 0\leq t\leq R.
\]
The right-hand side is continuous.  Uniqueness of the  integrand
shows that these representations agree on overlapping intervals, and thus they
define a single locally square-integrable process $\xi$ on $[0,T)$.
\end{proof}
\subsection{Proof of Theorem~\ref{thrm:rep}}

Conditional versions of \cite[Theorems 1 and 2]{ust_mart}, even under
sufficiently regular assumptions, require significant
technical care because the projection process $(P_t(X_t))_t$ need not be
continuous. In particular, switching between probability measures changes
the corresponding completed filtrations and therefore requires care when
passing to enlargements of the filtration generated by the diffusion
process. In the proof of Theorem~\ref{thrm:rep}, for this reason we use predictable $\sigma$-field associated with the raw filtration generated by $X$ to obtain the desired predictable representation property of $X$.

{The proof is organized into several parts. We start from a probability space $(\Omega,\mathcal F,\mathbb P)$ such that a weak solution $(X,B)$ to \eqref{eq:s} lies on. Then, we show that $(X,B)$ is also a solution to \eqref{eq:s} under the conditioned measure $\mathbb P_x(\,\cdot\,) = \mathbb P(\,\cdot\,\mid X_0=x).$ This requires $B$ to be a Brownian motion under $\mathbb P_x$ too. Proofs of these statements are relatively standard; however, since we were not able to find a source for these statements, we provide complete proofs.}

\begin{lemma}
\label{lem:fiber-brownian}
Let $(X,B)$ be defined on
$(\Omega,\mathcal F,(\mathcal G_t)_{t\leq T},\mathbb P)$, where
$B$ is a $d$-dimensional $(\mathcal G_t,\mathbb P)$-Brownian motion,
$X$ is continuous and $(\mathcal G_t)$-adapted, and
$X_0$ is $\mathcal G_0$-measurable. Let
$\mu_0:=\mathcal L_{\mathbb P}(X_0)$, and let
$(\mathbb P_x)_{x\in\mathbb R^n}$ be a regular conditional
probability of $\mathbb P$ given $X_0$. Define the raw $\sigma$-algebra
\begin{equation}
\mathcal G_t^{0}
:=
\sigma(X_s,B_s:0\leq s\leq t),
\end{equation}
and let $(\mathcal G_t^x)_{t\leq T}$ be the
$\mathbb P_x$-usual augmentation of $(\mathcal G_t^0)_{t\leq T}$.
Then, for $\mu_0$-a.e.\ $x$, $B$ is a $d$-dimensional
$(\mathcal G_t^x,\mathbb P_x)$-Brownian motion.
\end{lemma}

\begin{proof}
Since $X$ and $B$ are $(\mathcal G_t)$-adapted,
$\mathcal G_t^0\subseteq\mathcal G_t$ for every $t\leq T$.
Hence, under $\mathbb P$, every increment $B_t-B_s$ is independent of
$\mathcal G_s^0$.

For each rational $s\in[0,T]$, choose a countable $\pi$-system
$\mathcal C_s$ generating $\mathcal G_s^0$. Such a class may be
constructed from finite-dimensional cylinder events of
$(X_q,B_q)$, $q\in\mathbb Q\cap[0,s]$, since $X$ and $B$ have
continuous paths \cite{kuo2006gaussian}. Fix rational $0\leq s<t\leq T$,
$A\in\mathcal C_s$, and $\theta\in\mathbb Q^d$. For every bounded
Borel function $f:\mathbb R^n\to\mathbb R$, the random variable
$f(X_0)\mathbf 1_A$ is $\mathcal G_s$-measurable. Therefore,
\begin{equation}
\mathbb E^{\mathbb P}
\left[
f(X_0)\mathbf 1_A
e^{i\langle\theta,B_t-B_s\rangle}
\right]
=
e^{-\frac12|\theta|^2(t-s)}
\mathbb E^{\mathbb P}
\left[
f(X_0)\mathbf 1_A
\right].
\end{equation}
Disintegrating with respect to $X_0$ yields
\begin{align}
&\int_{\mathbb R^n}
f(x)
\mathbb E^{\mathbb P_x}
\left[
\mathbf 1_A
e^{i\langle\theta,B_t-B_s\rangle}
\right]
\mu_0(dx)
=
\int_{\mathbb R^n}
f(x)
e^{-\frac12|\theta|^2(t-s)}
\mathbb P_x(A)
\,\mu_0(dx).
\end{align}
Since this holds for every bounded Borel $f$,
\begin{equation}
\mathbb E^{\mathbb P_x}
\left[
\mathbf 1_A
e^{i\langle\theta,B_t-B_s\rangle}
\right]
=
e^{-\frac12|\theta|^2(t-s)}
\mathbb P_x(A)
\end{equation}
for $\mu_0$-a.e.\ $x$.

There are only countably many choices of
\begin{equation}
{(s,t,A,\theta) \in \mathbb Q\cap [0,T] \times \mathbb Q\cap [0,T] \times \mathcal C_s \times \mathbb Q^d.}
\end{equation}
Consequently, there exists a single $\mu_0$-null set $N$ such that,
for every $x\notin N$, all of the preceding identities hold
simultaneously. Enlarging $N$ if necessary, disintegration of the
$\mathbb P$-a.s.\ events $\{B_0=0\}$ and
$\{B\text{ has continuous paths}\}$ also gives, for every
$x\notin N$,
$
\mathbb P_x(B_0=0)=1
$
and $B$ has continuous paths $\mathbb P_x$-a.s.

Fix $x\notin N$. For rational $s<t$, the monotone class theorem
extends the preceding identity from $A\in\mathcal C_s$ to every
$A\in\mathcal G_s^0$. Continuity in $\theta$ then extends it from
$\theta\in\mathbb Q^d$ to every $\theta\in\mathbb R^d$. Thus,
\begin{equation}
\mathbb E^{\mathbb P_x}
\left[
\mathbf 1_A
e^{i\langle\theta,B_t-B_s\rangle}
\right]
=
\mathbb P_x(A)
e^{-\frac12|\theta|^2(t-s)},
\qquad
A\in\mathcal G_s^0.
\end{equation}
By uniqueness of characteristic functions,
$
B_t-B_s\sim\mathcal N(0,(t-s)I_d)
$
under $\mathbb P_x$, and $B_t-B_s$ is independent of
$\mathcal G_s^0$.

Now let $0\leq s<t\leq T$ be arbitrary. Choose rational sequences
$s_k\downarrow s$ and $t_k\to t$ such that $s_k<t_k$. If
$A\in\mathcal G_s^0$, then $A\in\mathcal G_{s_k}^0$, and hence
\begin{equation}
\mathbb E^{\mathbb P_x}
\left[
\mathbf 1_A
e^{i\langle\theta,B_{t_k}-B_{s_k}\rangle}
\right]
=
\mathbb P_x(A)
e^{-\frac12|\theta|^2(t_k-s_k)}.
\end{equation}
By continuity of $B$ and dominated convergence, letting
$k\to\infty$ gives
\begin{equation}
\mathbb E^{\mathbb P_x}
\left[
\mathbf 1_A
e^{i\langle\theta,B_t-B_s\rangle}
\right]
=
\mathbb P_x(A)
e^{-\frac12|\theta|^2(t-s)}.
\end{equation}
Thus $B_t-B_s$ is independent of $\mathcal G_s^0$ and has law
$\mathcal N(0,(t-s)I_d)$ for every $0\leq s<t\leq T$.

Finally, these properties are preserved under the
$\mathbb P_x$-usual augmentation of the raw filtration. Therefore,
$B$ is a $d$-dimensional
$(\mathcal G_t^x,\mathbb P_x)$-Brownian motion.
\end{proof}

Next, we show that a solution of $(X,B)$ of \eqref{eq:s} under $\mathbb P$ is also a solution under $\mathbb P_x$. This essentially amounts to picking predictable approximations of $(X,B)$ under the raw filtration, and showing that these approximations also work under $\mathbb P_x$

\begin{lemma}
\label{lem:fiber-sde}
Suppose, in addition to the hypotheses of Lemma~\ref{lem:fiber-brownian}, that
$b:[0,T]\times\mathbb R^n\to\mathbb R^n$ and
$\sigma:[0,T]\times\mathbb R^n\to\mathbb R^{n\times d}$ are bounded
Borel functions and that
\begin{equation}
X_t
=
X_0+\int_0^t b(s,X_s)\,ds
+\int_0^t\sigma(s,X_s)\,dB_s,
\qquad 0\leq t\leq T,
\end{equation}
holds $\mathbb P$-a.s. Then, for $\mu_0$-a.e.\ $x$,
\[
\mathbb P_x(X_0=x)=1,
\]
and, with respect to the filtration $(\mathcal G_t^x)_{t\leq T}$,
\begin{equation}
X_t
=
x+\int_0^t b(s,X_s)\,ds
+\int_0^t\sigma(s,X_s)\,dB_s,
\qquad 0\leq t\leq T,
\quad \mathbb P_x\text{-a.s.}
\end{equation}
Consequently, for $\mu_0$-a.e.\ $x$, $(X,B)$ is a weak solution of
\eqref{eq:s} with deterministic initial condition $X_0=x$ under
$\mathbb P_x$.
\end{lemma}

\begin{proof}
We first identify the initial condition. Since $(\mathbb P_x)_x$ is a
regular conditional probability of $\mathbb P$ given $X_0$, for every
bounded Borel function $f:\mathbb R^n\to\mathbb R$,
$
\mathbb E^{\mathbb P_x}[f(X_0)]
=
f(x)
$
for $\mu_0$-a.e.\ $x$. Taking $f$ from a countable determining class
for $\mathcal B(\mathbb R^n)$ and removing a single $\mu_0$-null set
gives
$
\mathcal L_{\mathbb P_x}(X_0)=\delta_x,
$
and hence
$
\mathbb P_x(X_0=x)=1
$
for $\mu_0$-a.e.\ $x$.

It remains to verify that the stochastic integral appearing in the
original $\mathbb P$-SDE agrees, on almost every fiber, with the It\^o
integral defined under $\mathbb P_x$. Since $X$ is continuous and adapted to the raw filtration
$(\mathcal G_t^0)_t$, the process $\sigma=(\sigma(t,X_t))_t$ is
$\mathcal G^0$-predictable \cite[Chapter IX, Proposition 1.1]{revuz2013continuous}. Since $\sigma$ is bounded, there exists a
sequence of bounded elementary $\mathcal G^0$-predictable
matrix-valued processes $(\sigma^m)_{m\geq1}$ such that, after passing
to a subsequence,
\begin{equation}
\sum_{m=1}^{\infty}
\mathbb E^{\mathbb P}
\int_0^T
\|\sigma^m(s,X_s)-\sigma(s,X_s)\|_{\mathrm F}^2\,ds
<\infty.
\label{eq:fiber-sigma-summable}
\end{equation}
For each $m$, define
$
I_t^m:=\int_0^t\sigma^m(s,X_s)\,dB_s.
$
Since $\sigma^m$ is {an} elementary predictable {process}, $I_t^m$ is given by a
finite sum of the form
\[
I_t^m
=
\sum_j
H_j^m
\bigl(B_{t\wedge t_{j+1}}-B_{t\wedge t_j}\bigr),
\]
and therefore represents the same random variable whether it is
viewed under $\mathbb P$ or under $\mathbb P_x$.

Let
\[
I_t:=\int_0^t\sigma(s,X_s)\,dB_s
\]
denote the It\^o integral under $\mathbb P$. By It\^o's isometry and
\eqref{eq:fiber-sigma-summable},
\begin{equation}
\sum_{m=1}^{\infty}
\mathbb E^{\mathbb P}
|I_t^m-I_t|^2
\leq
\sum_{m=1}^{\infty}
\mathbb E^{\mathbb P}
\int_0^T
\|\sigma_s^m-\sigma_s\|_{\mathrm F}^2\,ds
<\infty
\label{eq:fiber-int-P}
\end{equation}
for every $t\leq T$.

Disintegrating \eqref{eq:fiber-sigma-summable} with respect to $X_0$
and using Tonelli's theorem gives
\begin{align}
&\int_{\mathbb R^n}
\sum_{m=1}^{\infty}
\mathbb E^{\mathbb P_x}
\int_0^T
\|\sigma^m(s,X_s)-\sigma(s,X_s)\|_{\mathrm F}^2\,ds
\,\mu_0(dx)
 =
\sum_{m=1}^{\infty}
\mathbb E^{\mathbb P}
\int_0^T
\|\sigma^m(s,X_s)-\sigma(s,X_s)\|_{\mathrm F}^2\,ds
<\infty.
\end{align}
Hence, outside a $\mu_0$-null set,
\begin{equation}
\mathbb E^{\mathbb P_x}
\int_0^T
\|\sigma^m(s,X_s)-\sigma(s,X_s)\|_{\mathrm F}^2\,ds
\longrightarrow0.
\label{eq:fiber-integrand-conv}
\end{equation}

By Lemma~\ref{lem:fiber-brownian}, after enlarging the exceptional
set if necessary, $B$ is a
$(\mathcal G_t^x,\mathbb P_x)$-Brownian motion. Therefore,
$\sigma$ is square-integrable and predictable under $\mathbb P_x$.
It\^o's isometry under $\mathbb P_x$ and
\eqref{eq:fiber-integrand-conv} imply
\begin{equation}
I_t^m
\longrightarrow
I_t^x
:=
\int_0^t\sigma_s\,dB_s
\quad\text{in }L^2(\mathbb P_x)
\label{eq:fiber-Ix}
\end{equation}
for every $t\leq T$.

We now identify $I_t^x$ with the stochastic integral $I_t$ appearing
in the original $\mathbb P$-SDE. Fix
$t\in\mathbb Q\cap[0,T]$. Disintegrating
\eqref{eq:fiber-int-P} yields
\begin{align}
&\int_{\mathbb R^n}
\sum_{m=1}^{\infty}
\mathbb E^{\mathbb P_x}
|I_t^m-I_t|^2
\,\mu_0(dx)
=
\sum_{m=1}^{\infty}
\mathbb E^{\mathbb P}
|I_t^m-I_t|^2
<\infty.
\end{align}
Thus, for $\mu_0$-a.e.\ $x$,
\[
I_t^m\longrightarrow I_t
\quad\text{in }L^2(\mathbb P_x).
\]
Comparing this with \eqref{eq:fiber-Ix} and using uniqueness of the
$L^2(\mathbb P_x)$ limit gives
\begin{equation}
I_t
=
\int_0^t\sigma(s,X_s)\,dB_s
\qquad
\mathbb P_x\text{-a.s.}
\label{eq:fiber-integral-identification}
\end{equation}
for $\mu_0$-a.e.\ $x$. Since $\mathbb Q\cap[0,T]$ is countable, a
single $\mu_0$-null set can be chosen so that
\eqref{eq:fiber-integral-identification} holds simultaneously for all
rational $t\in[0,T]$.

For each rational $t\in[0,T]$, the original SDE identity
\[
X_t
=
X_0+\int_0^t b(s,X_s)\,ds+I_t
\]
holds $\mathbb P$-a.s. Disintegrating this probability-one event with
respect to $X_0$ and using countability of the rational $t$ gives,
outside one $\mu_0$-null set,
\[
X_t
=
x+\int_0^t b(s,X_s)\,ds
+\int_0^t\sigma(s,X_s)\,dB_s
\]
for every rational $t\in[0,T]$, $\mathbb P_x$-a.s., which follows from the definition of regular conditional probability directly.

{Finally, $X$ has a.s. continuous paths,}
$
t\longmapsto\int_0^t b(s,X_s)\,ds
$
is continuous, and
$
t\longmapsto\int_0^t\sigma(s,X_s)\,dB_s
$
has a continuous version under $\mathbb P_x$. Hence the preceding
{identity extends from rational $t$ to every $t\in[0,T]$. Since $X$}
is adapted to $(\mathcal G_t^x)_t$ and, by
Lemma~\ref{lem:fiber-brownian}, $B$ is a
$(\mathcal G_t^x,\mathbb P_x)$-Brownian motion, $(X,B)$ is a weak
solution of \eqref{eq:s} under $\mathbb P_x$ with initial condition
$x$.
\end{proof}

\begin{proof}[Proof of Theorem~\ref{thrm:rep}]
Let $(X,B)$ be the given weak solution and let
$(\mathbb P_x)_{x\in\mathbb R^n}$ be a regular conditional probability of
$\mathbb P$ given $X_0$:
\[
\mathbb P_x(\,\cdot\,) = \mathbb P(\,\cdot\,\mid X_0=x),
\]
whose existence follows from the running assumptions, see
\cite[Theorem~$I.3.3$]{ikeda2014stochastic}.

For each $x\in\mathbb R^n$, let $ \mathcal G_t^{0} := \sigma(X_s,B_s:0\leq s\leq t), $ and let $(\mathcal G_t^x)_{t\leq T}$ be its $\mathbb P_x$-usual augmentation. By Lemma~\ref{lem:fiber-brownian}, for $\mu_0$-a.e.\ $x$, $B$ is a $d$-dimensional $(\mathcal G_t^x,\mathbb P_x)$-Brownian motion. Moreover, by Lemma~\ref{lem:fiber-sde}, after removing a further $\mu_0$-null set, $ \mathbb P_x(X_0=x)=1 $ and \[ X_t = x+\int_0^t b(s,X_s)\,ds +\int_0^t\sigma(s,X_s)\,dB_s, \qquad 0\leq t\leq T, \quad \mathbb P_x\text{-a.s.} \] Therefore, for $\mu_0$-a.e.\ $x$, $(X,B)$ is a weak solution of \eqref{eq:s} with initial condition $x$ on $(\Omega,\mathcal F,(\mathcal G_t^x)_{t\leq T},\mathbb P_x)$.

{For any given probability measure $\mathbb Q$, write}
$\mathcal F^{\mathbb Q}(X)$ for the usual augmentation of the raw state
{filtration $(\mathcal F^0_t(X))_t$ under $\mathbb Q$.}

We argue by orthogonality as in \cite[Theorem 1]{ust_mart}, now
disintegrating with respect to $\mathbb P_x$.  Let $\mathcal K$ be the closed
subspace of $L^2(\mathcal F_T^{\mathbb P}(X),\mathbb P;\mathbb R)$ generated
by random variables of the
form 
\[
g(X_0)+\int_0^T \langle H_t\,,dB_t\rangle ,
\]
where $g(X_0)\in L^2(\mathcal F_0^{\mathbb P}(X))$, $H$ is
{$\mathcal F^{\mathbb P}(X)$-predictable, square-integrable w.r.t. Lebesgue measure on $[0,T]$, and}
$H_t=P_t(X)H_t$ $dt\otimes d\mathbb P$-a.e. 
It suffices to show $\mathcal K^\perp=\{0\}$, where $\mathcal K^\perp$ is the orthogonal complement of $\mathcal K$ under the $L^2$-norm.  Fix
$Z\in\mathcal K^\perp$.  Choose an $\mathcal F_T^0(X)$-measurable
representative equal to $Z$ $\mathbb P$-a.s.\ and continue to denote it by
$Z$.  Disintegration transfers this equality to $\mathbb P_x$-a.s.\ for
$\mu_0$-a.e.\ $x$; in particular, $Z$ is
$\mathcal F_T^{\mathbb P_x}(X)$-measurable on those fibers.  Since
$\mathcal F_0^{\mathbb P}(X)=\sigma(X_0)\vee\mathcal N_{\mathbb P}$,
orthogonality to every $g(X_0)$ gives
$\mathbb E^{\mathbb P}[Z\mid X_0]=0$.

Let $\mathscr P^0(X)$ denote the predictable $\sigma$-field associated
with the raw filtration $(\mathcal F_t^0(X))_{t\in[0,T]}$
\cite[Chapter I, Exercise 4.20]{revuz2013continuous}.  Since the continuous
path space $C([0,T];\mathbb R^n)$ is Polish, the raw state filtration $(\mathcal F_t(X))_t$ and
its predictable $\sigma$-field are countably generated \cite[Proposition  2.19]{MR2280298}.  Hence, we may
choose a countable family of bounded,
\(\mathbb R^d\)-valued, \(\mathscr P^0(X)\)-measurable simple processes
generated by a countable algebra (because the underlying $L^2$ space is separable).  We define \((J^m)_{m\ge1}\) to be an
enumeration of all rational finite linear combinations of this family.
Then, \((J^m)_{m\ge1}\) itself is dense in
\[
L^2\bigl(\mathscr P^0(X),dt\otimes \mathbb Q;\mathbb R^d\bigr)
\]
for every probability measure \(\mathbb Q\).

For \(m\ge1\), set
\[
H_t^m:=P_t(X)J_t^m.
\]
Since \(X\) is continuous and adapted to its raw filtration
\((\mathcal F_t^0(X))_{t\in[0,T]}\), it is
\(\mathcal F^0(X)\)-predictable. Since
\((t,x)\mapsto \sigma(t,x)\) is Borel measurable,
\((\sigma(t,X_t))_t\) is \(\mathscr P^0(X)\)-measurable.
Moreover, as argued in the proof of Lemma~\ref{lem:projection-measurable}, the map
\[
a\mapsto a^\top(aa^\top)^\dagger a
\]
is Borel measurable. Hence \((P_t(X))_t\) is
\(\mathscr P^0(X)\)-measurable (and predictable). Therefore, \(H^m\) is
\(\mathscr P^0(X)\)-measurable, and hence predictable with respect to
\(\mathcal F^{\mathbb P}(X)\). Since \(J^m\) is bounded and \(P_t(X)\)
is an orthogonal projection, \(H^m\) is square-integrable. Furthermore,
\[
P_t(X)H_t^m
=
P_t(X)^2J_t^m
=
P_t(X)J_t^m
=
H_t^m,
\qquad
dt\otimes d\mathbb P\text{-a.e.}
\]
For each $m$, since $H^m$ is $\mathscr P^0(X)$-measurable and
square-integrable, there exists a sequence of raw elementary
predictable processes $(H^{m,k})_{k\geq1}$ such that
\[
\mathbb E^{\mathbb P}\int_0^T
|H_t^{m,k}-H_t^m|^2\,dt
\longrightarrow0,
\]
see \cite[Proposition 4.6]{hairer2016advanced}.
Passing to a subsequence, we may assume
that
\[
\sum_{k=1}^{\infty}
\mathbb E^{\mathbb P}\int_0^T
|H_t^{m,k}-H_t^m|^2\,dt
<\infty.
\]

Disintegration and Tonelli's theorem then show, simultaneously for all
$m$, that $H^{m,k}\to H^m$ in
$L^2(dt\otimes d\mathbb P_x)$ for $\mu_0$-a.e.\ $x$.  Choose the
$\mathbb P$-version of $\int H^m\,dB$ as the $L^2(\mathbb P)$ limit of
the corresponding raw elementary integrals.  Applying the same summable
subsequence argument to the integral errors and using It\^o's isometry under
both $\mathbb P$ and $\mathbb P_x$ shows that this fixed version is also
the $\mathbb P_x$-It\^o integral for $\mu_0$-a.e.\ $x$.  Intersecting the
resulting countably many full-measure sets makes this choice valid for every
$m$ below.

Hence, for every Borel set \(A\subset\mathbb R^n\), $L^2$-orthogonality of
\(Z\) to \(\mathcal K\) gives
\[
0
=
\mathbb E^{\mathbb P}
\left[
Z\,1_{\{X_0\in A\}}
\int_0^T \langle H_t^m\,,dB_t\rangle 
\right] 
=
\mathbb E^{\mathbb P}
\left[
Z\,
\int_0^T \langle 1_{\{X_0\in A\}}\,H_t^m\,,dB_t\rangle
\right].
\]
Therefore,
\begin{equation}\label{eq:conditional-zero}
\mathbb E^{\mathbb P} \left[ Z\int_0^T \langle H_t^m\,,dB_t\rangle \,\middle|\,X_0 \right] 
=0
\implies 
\mathbb E^{\mathbb P_x} \left[Z\int_0^T \langle H_t^m\,,dB_t\rangle\right]=0.
\end{equation}
Since the family \((H^m)_{m\ge1}\) is countable, there exists a
single set \(N\subset\mathbb R^n\), with \(\mu_0(N)=0\), such that
the preceding equality holds simultaneously for every \(m\ge1\)
and every \(x\notin N\).

Now fix $x\notin N$ such that
$\mathbb E^{\mathbb P_x}[Z]=0$, $Z\in L^2(\mathbb P_x)$, and the
fixed-initial-condition hypotheses hold.  Let
$H\in L_a^2\bigl(\mathcal F^{\mathbb P_x}(X),\mathbb P_x;\mathbb R^d\bigr)$
satisfy
$
P_t(X)H_t=H_t
$
$dt\otimes d\mathbb P_x$-a.e.,
where $(P_t(X))_t$ is $\mathcal F^{\mathbb P_x}(X)$-adapted by Lemma \ref{lem:projection-measurable}.
Since the predictable $\sigma$-field of the augmented filtration is the
$dt\otimes\mathbb P_x$-completion of $\mathscr P^0(X)$, the process $H$
admits a $\mathscr P^0(X)$-measurable representative. In particular, by density of
\((J^m)_{m\ge1}\), there exists a sequence \((J^{m_k})_{k\ge1}\)
such that
\[
\mathbb E^{\mathbb P_x}
\int_0^T
\left|J_t^{m_k}-H_t\right|^2\,dt
\longrightarrow0.
\]
Since \(P_t(X)\) is an orthogonal projection and
\(P_t(X)H_t=H_t\) $dt\otimes d\mathbb P_x$-a.e.,
\[
\begin{aligned}
\mathbb E^{\mathbb P_x}
\int_0^T
\left|H_t^{m_k}-H_t\right|^2\,dt
&=
\mathbb E^{\mathbb P_x}
\int_0^T
\left|
P_t(X)J_t^{m_k}-P_t(X)H_t
\right|^2\,dt \\
&\le
\mathbb E^{\mathbb P_x}
\int_0^T
\left|J_t^{m_k}-H_t\right|^2\,dt
\longrightarrow0.
\end{aligned}
\]
By It\^o's isometry under \(\mathbb P_x\),
\[
\int_0^T \langle H_t^{m_k}\,,dB_t\rangle 
\longrightarrow
\int_0^T \langle H_t\,,dB_t \rangle
\qquad
\text{in }L^2(\mathbb P_x).
\]
Since \(Z\in L^2(\mathbb P_x)\), by Cauchy-Schwarz and It\^o isometry, it follows that
\begin{equation}\label{eq:fiber-orthogonality}
\mathbb E^{\mathbb P_x}
\left[
Z\int_0^T \langle H_t\,,dB_t \rangle 
\right]
=
\lim_{k\to\infty}
\mathbb E^{\mathbb P_x}
\left[
Z\int_0^T \langle H_t^{m_k}\,,dB_t \rangle
\right]
=0.
\end{equation}
Thus $Z$ is orthogonal under $\mathbb P_x$ to every stochastic integral
$\int_0^T\langle H_s,dB_s\rangle$ with
$H\in L_a^2(\mathcal F^{\mathbb P_x}(X),\mathbb P_x;\mathbb R^d)$ and
$P_t(X)H_t=H_t$ $dt\otimes d\mathbb P_x$-a.e.

Under $\mathbb P_x$, $X_0=x$.  The fixed-initial representation theorem
\cite[Theorem 1]{ust_mart}, together with
\eqref{eq:conditional-zero}--\eqref{eq:fiber-orthogonality}, therefore gives
$Z=0$ $\mathbb P_x$-a.s.
for $\mu_0$-a.e.\ $x$.
Integrating with respect to \(\mu_0\),
\[
\mathbb P(Z\neq0)
=
\int
\mathbb P_x(Z\neq0)\,\mu_0(dx)
=0.
\]
Thus $\mathcal K^\perp=\{0\}$. It remains to show that closure of $\mathcal K$ in $L^2$ consist of terms of the form
\[
g(X_0) + \int_0^T \langle H_s,dB_s\rangle,
\]
where $P_t(X)H_t = H_t$ $dt \otimes d\mathbb P$-a.e.

To show this, we closely follow the proof of \cite[Theorem 2]{ust_mart}.

Let $
\widetilde F
:=
F-\mathbb E^{\mathbb P}[F\mid X_0],
$
i.e.,
$
\mathbb E^{\mathbb P}
\bigl[\widetilde F\mid X_0\bigr]=0.
$
By the preceding density argument, there exists a sequence
\((F^m)_{m\geq1}\subset L^2(\mathcal F_T^{\mathbb P}(X),\mathbb P)\)
such that
$
F^m\rightarrow \widetilde F
$ 
in $L^2(\mathbb P),$
and, for every $m$,
\[
F^m
=
\int_0^T\langle H_t^m,dB_t\rangle,
\qquad
H_t^m=P_t(X)H_t^m
\quad dt\otimes d\mathbb P\text{-a.e.}
\]
Since \(F^m\to\widetilde F\) in \(L^2(\mathbb P)\), It\^o's isometry gives
\[
\lim_{m,k\to\infty}
\mathbb E^{\mathbb P}
\int_0^T
|H_t^m-H_t^k|^2\,dt
=
\lim_{m,k\to\infty}
\mathbb E^{\mathbb P}
|F^m-F^k|^2
=
0.
\]
Hence \((H^m)_{m\geq1}\) converges in
\(L^2(dt\otimes d\mathbb P;\mathbb R^d)\) to some predictable process
\(H\). Since \(P_t(X)\) is an orthogonal projection and
\(P_t(X)H_t^m=H_t^m\), we also have
$
P_t(X)H_t=H_t
$ $dt\otimes d\mathbb P$-a.e.
Indeed,
\[
\mathbb E^{\mathbb P}\int_0^T
|P_t(X)H_t-H_t|^2\,dt
\leq
4\lim_{m\to\infty}
\mathbb E^{\mathbb P}\int_0^T
|H_t-H_t^m|^2\,dt
=0.
\]
Applying It\^o's isometry once more,
\[
\int_0^T\langle H_t^m,dB_t\rangle
\longrightarrow
\int_0^T\langle H_t,dB_t\rangle
\qquad\text{in }L^2(\mathbb P).
\]
Therefore,
\[
\widetilde F
=
\int_0^T\langle H_t,dB_t\rangle
\qquad\mathbb P\text{-a.s.},
\]
and consequently
\[
F
=
\mathbb E^{\mathbb P}[F\mid X_0]
+
\int_0^T\langle H_t,dB_t\rangle,
\qquad
H_t=P_t(X)H_t
\quad dt\otimes d\mathbb P\text{-a.e.}
\]
This proves the assertion.
\end{proof}

\subsection{Proof of Theorem \ref{thrm}}

\begin{proof}[Proof of Theorem \ref{thrm}]
Since $X_0$ is $\mathcal G_0$-measurable and $B$ is a
$(\mathcal G_t)$-Brownian motion,
\[
X_0\ \perp\!\!\!\perp\ \sigma(B_s:0\leq s\leq T).
\]
Thus, the enlargement of the Brownian filtration by $X_0$ is independent,
and the equivalence condition in
\cite[Theorem 4.2]{amendinger2000martingale} is satisfied.  Consequently,
\eqref{bm:rep} follows in $(\mathcal H_t)_{t\leq T}$; equivalently, one may
apply the Brownian martingale representation theorem conditionally on
$X_0$.  For the rest of the proof, we follow the projection argument of
\cite[Theorem 3]{ust_mart}.

Let $\widehat\xi$ be the $L^2(dt\otimes d\mathbb P)$ predictable
projection of $\xi$ onto the closed subspace of
$\mathcal F(X)$-predictable processes \cite[Exercise 5.14]{revuz2013continuous}.  Then,
\[
\mathbb E^{\mathbb P}\int_0^T |\widehat\xi_t|^2dt
\leq
\mathbb E^{\mathbb P} \int_0^T |\xi_t|^2dt<\infty,
\]
and, for every square-integrable
$\mathcal F(X)$-predictable process $H$,
\begin{equation}\label{eq:adapted-projection-xi}
\mathbb E^{\mathbb P}
\int_0^T
\left\langle
\xi_t-\widehat\xi_t,H_t
\right\rangle dt
=0.
\end{equation}

Define $K_t:=P_t(X)\widehat\xi_t$.
By Lemma~\ref{lem:projection-measurable}, $P_t(X)$ is
$\mathcal F(X)$-predictable. Hence $K$ is an
$\mathcal F(X)$-predictable process. Moreover, since $P_t(X)$ is an
orthogonal projection with respect to the Euclidean metric, we have
$ |K_t| \leq |\widehat\xi_t|$ $dt \otimes \mathbb P$-a.e.,
and therefore $K$ is square-integrable. Also,
\begin{equation}\label{eq:hor}
P_t(X)K_t = P_t(X)^2\widehat\xi_t = K_t
\qquad
dt\otimes d\mathbb P\text{-a.e.}
\end{equation}

Set
\[
\widetilde F
:=
\mathbb E^{\mathbb P}[F\mid X_0]
+
\int_0^T\langle K_t,dB_t\rangle.
\]
By Theorem~\ref{thrm:rep}, since $K$ is
$\mathcal F(X)$-predictable, square-integrable, and \eqref{eq:hor} holds,
we have $\widetilde F\in L^2(\mathcal F_T(X),\mathbb P).$
We claim that
$\widetilde F = \mathbb E^{\mathbb P}[F\mid\mathcal F_T(X)].$
It is therefore enough to prove that
\[
\mathbb E^{\mathbb P}[(F-\widetilde F)Z]=0
\qquad
\text{for every }
Z\in L^2(\mathcal F_T(X),\mathbb P;\mathbb R).
\]

Fix such a $Z$. By Theorem~\ref{thrm:rep}, there exists a
square-integrable $\mathcal F(X)$-predictable process $H$ satisfying
$
P_t(X)H_t=H_t
$
$dt\otimes d\mathbb P$-a.e. such that
\begin{equation}\label{eq:Z--representation}
Z
=
\mathbb E^{\mathbb P}[Z\mid X_0]
+
\int_0^T\langle H_t,dB_t\rangle.
\end{equation}

On the other hand, by the initially enlarged-filtration Brownian martingale
representation established above,
\[
F = \mathbb E^{\mathbb P}[F\mid X_0] + \int_0^T\langle\xi_t,dB_t\rangle \implies F-\widetilde F = \int_0^T \langle \xi_t-K_t,dB_t\rangle.
\]
In particular, $\mathbb E^{\mathbb P}[F-\widetilde F\mid X_0]=0$ $\mathbb P$-a.e.

Using \eqref{eq:Z--representation}, we therefore obtain
\[
\begin{aligned}
\mathbb E^{\mathbb P}[(F-\widetilde F)Z]
&=
\mathbb E^{\mathbb P}\left[
(F-\widetilde F)\mathbb E^{\mathbb P}[Z\mid X_0]
\right]
\\
&\quad+
\mathbb E^{\mathbb P}\left[
\left(
\int_0^T
\langle\xi_t-K_t,dB_t\rangle
\right)
\left(
\int_0^T
\langle H_t,dB_t\rangle
\right)
\right].
\end{aligned}
\]
The first term is zero by conditioning on $X_0$. By It\^o's isometry, we have
\[
\mathbb E^{\mathbb P}\left[
\left(
\int_0^T
\langle\xi_t-K_t,dB_t\rangle
\right)
\left(
\int_0^T
\langle H_t,dB_t\rangle
\right)
\right]=
\mathbb E^{\mathbb P}
\int_0^T
\langle \xi_t-K_t,H_t\rangle dt.
\]
Since $K_t=P_t(X)\widehat\xi_t$, the symmetry of the orthogonal
projection $P_t(X)$ and the invariance of $H_t$ under $P_t(X)$ yield
\begin{align*}
\langle K_t,H_t\rangle
&=
\left\langle
P_t(X)\widehat\xi_t,H_t
\right\rangle
\\
&=
\left\langle
\widehat\xi_t,P_t(X)H_t
\right\rangle
\\
&=
\langle\widehat\xi_t,H_t\rangle.
\end{align*}
Consequently,
\[
\mathbb E^{\mathbb P}[(F-\widetilde F)Z]
=
\mathbb E^{\mathbb P}
\int_0^T
\langle\xi_t-\widehat\xi_t,H_t\rangle dt.
\]
Since $H$ is $\mathcal F(X)$-predictable, the defining
orthogonality property of \eqref{eq:adapted-projection-xi} gives
\[
\mathbb E^{\mathbb P}[(F-\widetilde F)Z]=0.
\]
Since this holds for every
$Z\in L^2(\mathcal F_T(X))$ and
$\widetilde F\in L^2(\mathcal F_T(X))$, the Hilbert-space
characterization of conditional expectation implies
$
\widetilde F
=
\mathbb E^{\mathbb P}[F\mid\mathcal F_T(X)]
$ $\mathbb P$-a.s.
Therefore,
\[
\mathbb E^{\mathbb P}[F\mid\mathcal F_T(X)]
=
\mathbb E^{\mathbb P}[F\mid X_0]
+
\int_0^T
\left\langle
P_t(X)\widehat\xi_t,
dB_t
\right\rangle.
\]
This proves the result.
\end{proof}

\subsection{Proof of Theorem \ref{thm:main}}

\begin{proof}[Proof of Theorem \ref{thm:main}]
Fix $r<T$ and choose $R$ such that $r<R<T$.  By
Lemma~\ref{lem:density-martingale}, on $[0,R]$,
\[
Z_t(X)=Z_0(X)+\int_0^t\langle\xi_s,dB_s\rangle,
\qquad P_s(X)\xi_s=\xi_s.
\]
Set
\[
u_t:=\frac{\xi_t}{Z_t(X)}
\quad\text{on }\{Z_t(X)>0\},
\]
and set $u_t=0$ on $\{Z_t(X)=0\}$.  This set is
$dt\otimes d\mathbb P^y$-null because
$\mathbb P^y(Z_t(X)=0)=\mathbb E^{\mathbb P}
[Z_t(X)\mathbf1_{\{Z_t(X)=0\}}]=0$.

For $n\geq1$, define the stopping time
\[
\tau_n^R
:=
R\wedge\inf\left\{t\geq0:
Z_t(X)\notin\left(\frac1n,n\right)\right\},
\]
with $\inf\varnothing=\infty$.  On $[0,\tau_n^R]$,
$1/n\leq Z_t(X)\leq n$, and
\[
dZ_t(X)=Z_t(X)\langle u_t,dB_t\rangle.
\]
The change of measure separates into an
initial change by the $\mathcal G_0$-measurable density $Z_0$ and a dynamic
change generated by $u$.  The initial change leaves Brownian increments
Brownian, since they are independent of $\mathcal G_0$; on
$\{Z_0>0\}$, the normalized density $Z_t/Z_0$ generates the dynamic
Girsanov shift.  Thus, Girsanov's theorem over $(\mathcal F_t(X))_t$ shows that
\[
M_{t\wedge\tau_n^R}
:=
\int_0^{t\wedge\tau_n^R}P_s(X)\,dB_s
-\int_0^{t\wedge\tau_n^R}u_s\,ds
\]
is an $(\mathcal F_t(X),\mathbb P^y)$-local martingale.  It is in fact
square-integrable, since
\[
\mathbb E^{\mathbb P^y}\int_0^{\tau_n^R}|u_s|^2ds
=
\mathbb E^{\mathbb P}\int_0^{\tau_n^R}
\frac{|\xi_s|^2}{Z_s(X)}\,ds
\leq
n\,\mathbb E^{\mathbb P}\int_0^R|\xi_s|^2ds<\infty.
\]

We next remove the localization on the fixed interval $[0,R]$.  Since
$Z$ is a nonnegative continuous martingale, it remains zero after its
first hitting time of zero.  Therefore, for
\[
A_R:=\left\{\inf_{0\leq s\leq R}Z_s(X)=0\right\},
\]
we have $Z_R(X)=0$ on $A_R$, and thus
\[
\mathbb P^y(A_R)=\mathbb E^{\mathbb P}[Z_R(X)\mathbf1_{A_R}]=0.
\]
On $A_R^c$, continuity gives
$0<\inf_{s\leq R}Z_s(X)\leq\sup_{s\leq R}Z_s(X)<\infty$.
On $A_R$, the stopped values converge to zero, which equals $Z_R$ because a
nonnegative continuous martingale remains zero after its first hitting time
of zero.  Consequently, pointwise,
\[
Z_{\tau_n^R}(X)\longrightarrow Z_R(X)
\qquad \mathbb P\text{-a.s.},
\]
and $\tau_n^R\nearrow R$ $\mathbb P^y$-a.s.

Moreover, It\^o's formula for $\log Z$ on $[0,\tau_n^R]$ gives
\[
\frac12\mathbb E^{\mathbb P^y}
\int_0^{\tau_n^R}|u_s|^2ds
=
\mathbb E^{\mathbb P^y}
\left[
\log\frac{Z_{\tau_n^R}(X)}{Z_0(X)}
\right]
=
\mathbb E^{\mathbb P}
\left[
Z_{\tau_n^R}(X)\log Z_{\tau_n^R}(X)
\right]
-
\mathbb E^{\mathbb P}
\left[
Z_0(X)\log Z_0(X)
\right].
\]
By Doob's Optional Sampling Theorem, we obtain
$Z_{\tau_n^R}=\mathbb E^{\mathbb P}[Z_R\mid
\mathcal F_{\tau_n^R}(X)]$. In particular, conditional Jensen's inequality gives
\[
Z_{\tau_n^R}^2
\leq
\mathbb E^{\mathbb P}
[Z_R^2\mid\mathcal F_{\tau_n^R}(X)].
\]
The conditional expectations on the right form a uniformly integrable
family.  Hence $(Z_{\tau_n^R}^2)_n$ is uniformly integrable.  Since
$(z\log z)^+\leq z^2$ and $(z\log z)^-\leq e^{-1}$, the family
$(Z_{\tau_n^R}\log Z_{\tau_n^R})_n$ is uniformly integrable as well.
Together with the almost-sure convergence above, by Vitali's convergence theorem,
\[
\mathbb E^{\mathbb P}
[Z_{\tau_n^R}\log Z_{\tau_n^R}]
\longrightarrow
\mathbb E^{\mathbb P}[Z_R\log Z_R].
\]
Using monotone convergence on $\int_0^{\tau_n^R}|u_s|^2ds = \int_0^{\infty}|u_s|^21_{\{s \le \tau_n^R\}}ds $, we obtain $\lim_{n \to \infty} \int_0^{\tau_n^R}|u_s|^2ds= \int_0^{R}|u_s|^2ds$. Thus, we obtain both square integrability and the exact identity
\begin{equation}\label{eq:exact-localized-entropy}
\frac12\mathbb E^{\mathbb P^y}\int_0^R|u_s|^2ds
=
\mathbb E^{\mathbb P}[Z_R\log Z_R]
-\mathbb E^{\mathbb P}[Z_0\log Z_0].
\end{equation}
Since $\tau_n^R\nearrow R$ $\mathbb P^y$-a.s. and
$M^{\tau_n^R}$ is a martingale for every $n$, the process
\[
M_t
:=
\int_0^t P_s(X)\,dB_s
-
\int_0^t u_s\,ds,
\qquad 0\leq t\leq R,
\]
is a continuous $(\mathcal F_t(X),\mathbb P^y)$-local martingale.
Its matrix quadratic variation is
\[
\langle M\rangle_t
=
\int_0^tP_s(X)P_s(X)^\top\,ds
=\int_0^tP_s(X)\,ds \implies 
\operatorname{tr}\langle M\rangle_R
=
\int_0^R \operatorname{tr}\bigl(P_s(X)\bigr)\,ds
\leq d\,R.
\]
Hence, by \cite[Chapter IV, Proposition 1.23]{revuz2013continuous}
(applied componentwise), $M$ is a square-integrable
$(\mathcal F_t(X),\mathbb P^y)$-martingale on $[0,R]$.

{Define the {change of} measure $\mathbb Q^{y,R}$ obtained on}
$\mathcal G_R$ by
\[
\frac{d(\mathbb Q^{y,R}|_{\mathcal G_R})}
     {d(\mathbb P|_{\mathcal G_R})}=Z_R(X).
\]
The representation of $Z$ is a square-integrable
$(\mathcal G_t,\mathbb P)$-martingale, and thus $Z_t(X)$ is the density process
of $\mathbb Q^{y,R}$.  Moreover,
$\mathbb Q^{y,R}|_{\mathcal F_R(X)}
=\mathbb P^y|_{\mathcal F_R(X)}$.  The preceding localization and
Girsanov's theorem show that
\[
B_t^{y,R}:=B_t-\int_0^t u_s\,ds,
\qquad 0\leq t\leq R,
\]
is a $(\mathcal G_t,\mathbb Q^{y,R})$-Brownian motion.  Since
$\mathbb Q^{y,R}\ll\mathbb P$ on $\mathcal G_R$, the original SDE
formulation (\eqref{eq:s}) is preserved.  Substitution gives
\begin{equation}\label{eq:weak}
dX_t
=\bigl(b(t,X_t)+\sigma(t,X_t)u_t\bigr)dt
+\sigma(t,X_t)\,dB_t^{y,R},
\qquad 0\leq t\leq R.
\end{equation}
Consequently,
\[
\mathcal L_{\mathbb Q^{y,R}}(X_{[0,R]})
=\mathcal L_{\mathbb P}(X_{[0,R]}\mid Y=y).
\]

For $0<\varepsilon<R-r$ and $0\leq t\leq r$, the martingale property
of $M$ together with Theorem~\ref{thrm} give
\begin{align*}
\mathbb E^{\mathbb P^y} \left[
\left.\frac1\varepsilon\int_t^{t+\varepsilon}
P_s(X)\,dB_s\right|X_t\right]  =
\mathbb E^{\mathbb P^y}\left[
\left.\frac1\varepsilon\int_t^{t+\varepsilon}
u_s\,ds\right| X_t\right].
\end{align*}
Assumption~\ref{ass:feedback} gives the
$\sigma(X_t)$-measurable version $u_t=u_t(X_t,y)$.
As $u\in L^2([0,R]\times\Omega,dt\otimes d\mathbb P^y)$, by the Lebesgue differentiation theorem,
\[
\frac1\varepsilon\int_t^{t+\varepsilon}u_s\,ds
\longrightarrow u_t
\quad\text{in }
L^2([0,r]\times\Omega,dt\otimes d\mathbb P^y).
\]
Since $u_t$ is $\sigma(X_t)$-measurable and conditional expectation is an
$L^2$-contraction, the same convergence holds after conditioning.
Therefore
\[
\lim_{\varepsilon\downarrow0}
\mathbb E^{\mathbb P^y}\left[
\left.\frac1\varepsilon\int_t^{t+\varepsilon}
P_s(X)\,dB_s\right| X_t\right]
=u_t
=\frac{\xi_t}{Z_t(X)}
\]
in $L^2([0,r]\times\Omega,dt\otimes d\mathbb P^y)$.  Since $r<T$ was
arbitrary, the result follows on every compact subinterval of $[0,T)$, as desired.
\end{proof}

\subsection{Proof of Theorem \ref{thrm:causal}}

\begin{proof}[Proof of Theorem \ref{thrm:causal}]
Fix $r<T$. Throughout the proof, let
$ (\Omega,\mathcal F,(\mathcal G_t)_{0\leq t\leq r},\mathbb P) $
be the filtered probability space carrying a
weak solution $(X,B)$ of \eqref{eq:s}, where $B$ is a $(\mathcal G_t)$-Brownian
motion and
\[
\mathcal F_t(X)\subseteq\mathcal G_t,
\qquad 0\leq t\leq r.
\]
Thus, under $\mathbb P$,
\begin{equation}\label{eq:causal-original}
dX_t=b(t,X_t)\,dt+\sigma(t,X_t)\,dB_t.
\end{equation}

Let
\[
Z_t(X)
=
\frac{d(\mathbb P^y|_{\mathcal F_t(X)})}
     {d(\mathbb P|_{\mathcal F_t(X)})},
\qquad
u_t:=\frac{\xi_t}{Z_t(X)}
\quad\text{on }\{Z_t(X)>0\}.
\]
By Lemma~\ref{lem:density-martingale} and Theorem~\ref{thrm:rep}, there exists a square-integrable
 integrand $\xi=(\xi_s)_s$ on $[0,r]$ such that
\[
Z_t(X)
=
Z_0(X)+\int_0^t\langle \xi_s,dB_s\rangle, \qquad
Z_0(X)=\frac{d(\mathbb P^y|_{\mathcal F_0(X)})}
{d(\mathbb P|_{\mathcal F_0(X)})},
\qquad
P_s(X)\xi_s=\xi_s.
\]
In particular, since the integrand is $(\mathcal G_t)$-adapted,
$(Z_t(X))_{0\leq t\leq r}$ is also a $(\mathcal G_t,\mathbb P)$-
martingale.

{We now construct a probability measure under which the state process has the desired conditional law}. Define $\mathbb Q^{y,r}$ on
$(\Omega,\mathcal G_r)$ by
\begin{equation}\label{eq:ambient-y-measure}
\frac{d(\mathbb Q^{y,r}|_{\mathcal G_r})}
     {d(\mathbb P|_{\mathcal G_r})}
:=
Z_r(X).
\end{equation}
Since $Z_r(X)$ is $\mathcal F_r(X)$-measurable and
$\mathbb E^{\mathbb P}[Z_r(X)]=1$, this defines a probability measure.
Moreover, for every $A\in\mathcal F_r(X)$,
\[
\mathbb Q^{y,r}(A)
=
\mathbb E^{\mathbb P}[Z_r(X)\mathbf 1_A]
=
\mathbb P^y(A).
\]
Hence
\begin{equation}\label{eq:ambient-restriction}
\mathbb Q^{y,r}|_{\mathcal F_r(X)}
=
\mathbb P^y|_{\mathcal F_r(X)}.
\end{equation}

Furthermore, because $Z$ is a $(\mathcal G_t,\mathbb P)$-martingale
with terminal value $Z_r(X)$,
\[
Z_t(X)
=
\mathbb E^{\mathbb P}
\left[
Z_r(X)\mid\mathcal G_t
\right],
\qquad 0\leq t\leq r.
\]
Thus, $Z$ is the Radon-Nikodym derivative of $\mathbb Q^{y,r}$ with respect
to $\mathbb P$ on $(\mathcal G_t)_t$.

As proved by the localization argument in Theorem \ref{thm:main},
under $\mathbb Q^{y,r}$, the process
\[
B_t^{y,r}
:=
B_t-\int_0^t u_s\,ds,
\qquad 0\leq t\leq r,
\]
is a $(\mathcal G_t)$-Brownian motion. Since
\eqref{eq:causal-original} holds $\mathbb P$-a.s. and
$\mathbb Q^{y,r}\ll\mathbb P$, the same stochastic identity holds
$\mathbb Q^{y,r}$-a.s. Substituting
$dB_t=dB_t^{y,r}+u_tdt$ gives
\begin{equation}\label{eq:canonical-controlled-sde}
dX_t
=
\bigl(
b(t,X_t)+\sigma(t,X_t)u_t
\bigr)dt
+
\sigma(t,X_t)dB_t^{y,r}.
\end{equation}
Consequently,
\[
\bigl(
\Omega,\mathcal F,(\mathcal G_t)_{t\leq r},
\mathbb Q^{y,r},X,B^{y,r}
\bigr)
\]
is a weak solution of \eqref{eq:s2} with control $u$.

By \eqref{eq:ambient-restriction},
\[
\mathcal L_{\mathbb Q^{y,r}}(X_{[0,r]})
=
\mathcal L_{\mathbb P^y}(X_{[0,r]})
=
\mathcal L_{\mathbb P}(X_{[0,r]}\mid Y=y).
\]
{Thus, the solution of the controlled equation we have obtained via Girsanov shift has the conditioned state-path law $\mathcal L_{\mathbb P}(X_{[0,r]}\mid Y=y)$.}

We next prove its minimum-energy property.  Applying the exact limit identity
\eqref{eq:exact-localized-entropy} with $R=r$ yields
\begin{equation}\label{eq:canonical-entropy}
D_{\mathrm{KL}}
\left(
\mathbb P^y|_{\mathcal F_r(X)}
\,\middle\Vert\,
\mathbb P|_{\mathcal F_r(X)}
\right)
= D_{\mathrm{KL}}
\left(
\mathbb P^y|_{\mathcal F_0(X)}
\,\middle\Vert\,
\mathbb P|_{\mathcal F_0(X)}
\right)+
\frac12
\mathbb E^{\mathbb P^y}
\int_0^r|u_t|^2dt.
\end{equation}
Since $\int_0^r|u_t|^2dt$ is $\mathcal F_r(X)$-measurable and
$\mathbb Q^{y,r}$ and $\mathbb P^y$ agree on $\mathcal F_r(X)$,
equivalently,
\[
\frac12
\mathbb E^{\mathbb Q^{y,r}}
\int_0^r|u_t|^2dt
=
D_{\mathrm{KL}}
\left(
\mathbb P^y|_{\mathcal F_r(X)}
\,\middle\Vert\,
\mathbb P|_{\mathcal F_r(X)}
\right)
-
D_{\mathrm{KL}}
\left(
\mathbb P^y|_{\mathcal F_0(X)}
\,\middle\Vert\,
\mathbb P|_{\mathcal F_0(X)}
\right).
\]

{Let $(\mathbb Q,v)\in\mathcal A_r^y$ be such that $\mathcal L_{\mathbb Q}(X_{[0,r]}) = \mathcal L_{\mathbb P}(X_{[0,r]}|Y=y)$. Then}
$\mathbb Q|_{\mathcal F_r(X)}=\mathbb P^y|_{\mathcal F_r(X)}$, and
data processing inequality gives
\begin{align*}
D_{\mathrm{KL}}
\left(
\mathbb P^y|_{\mathcal F_r(X)}
\,\middle\Vert\,
\mathbb P|_{\mathcal F_r(X)}
\right)
&\leq D_{\mathrm{KL}}(\mathbb Q\Vert\mathbb P)\\
&=H_0^y+\frac12\mathbb E^{\mathbb Q}
\int_0^r|v_t|^2dt.
\end{align*}
Combining this inequality with \eqref{eq:canonical-entropy} yields
\[
\mathbb E^{\mathbb Q}\int_0^r|v_t|^2dt
\geq
\mathbb E^{\mathbb P^y}\int_0^r|u_t|^2dt.
\]

{The pair $(\mathbb Q^{y,r},u)$ has initial density $Z_0(X)$, and the construction}
above shows that $B^{y,r}=B-\int_0^\cdot u_s\,ds$ is Brownian under
$\mathbb Q^{y,r}$; indeed,
$\mathbb E^{\mathbb P}[Z_r(X)\mid\mathcal G_0]=Z_0(X)$.
Identity present in \eqref{eq:canonical-entropy} gives the
required entropy identity, and thus
$(\mathbb Q^{y,r},u)\in\mathcal A_r^y$ and it attains equality.

{If another admissible adapted Girsanov shift attains the same value, then the}
data-processing inequality used above must hold with equality.  By the
equality case of conditional Jensen's inequality,
\[
\frac{d\mathbb Q}{d\mathbb P}
=
\mathbb E^{\mathbb P}\left[
\left.
\frac{d\mathbb Q}{d\mathbb P}
\right|\mathcal F_r(X)
\right]
\qquad \mathbb P\text{-a.s.}
\]
Thus, \(d\mathbb Q/d\mathbb P\) admits an
\(\mathcal F_r(X)\)-measurable version.  Since its restriction to
\(\mathcal F_r(X)\) has density \(Z_r(X)\), it follows that
\[
\frac{d\mathbb Q}{d\mathbb P}
=
Z_r(X)
\qquad \mathbb P\text{-a.s.},
\]
and hence
\[
\mathbb Q=\mathbb Q^{y,r}
\qquad\text{on }\mathcal G_r.
\]
{Theorem~\ref{thrm:rep} was used above to identify the Girsanov shift}
\(u\) through the martingale representation of the density process, and
the corresponding Girsanov argument shows that
\(B-\int_0^\cdot u_s\,ds\) is a Brownian motion under
\(\mathbb Q^{y,r}\).  By admissibility,
\(B-\int_0^\cdot v_s\,ds\) is a Brownian motion under \(\mathbb Q\).
Under the common measure \(\mathbb Q=\mathbb Q^{y,r}\), uniqueness of
the continuous semimartingale decomposition therefore gives
\[
v_t=u_t
\qquad
dt\otimes d\mathbb Q^{y,r}\text{-a.e.}
\]
This proves the minimum-energy property and uniqueness.

\end{proof}

\section{Statement and Proof of Theorem~\ref{thm:degenerate-stochastic-control}}\label{app:d}

First, we provide a formal version of Theorem~\ref{thm:degenerate-stochastic-control}. For simplicity, we assume that $X_0$ is deterministic.

\begin{theorem}\label{thrm:10}
Suppose that $X_0=x_0$ is deterministic and Assumption~\ref{ass:1} holds.
Fix $r<T$, set $\mu_r:=\mathcal L_{\mathbb P}(X_r)$, and let $\nu_r$ be a
probability measure such that
\[
\nu_r\ll\mu_r,
\qquad
z_r:=\frac{d\nu_r}{d\mu_r}\in L^2(\mu_r).
\]
Let $\mathcal A_r$ be the class of pairs $(\mathbb Q,v)$ such that
$\mathbb Q\ll\mathbb P$ on $\mathcal G_r$, $v$ is
$(\mathcal G_t)$-predictable,
\[
B_t^{\mathbb Q}:=B_t-\int_0^t v_s\,ds
\]
is a $(\mathcal G_t,\mathbb Q)$-Brownian motion,
$\mathbb E^{\mathbb Q}\int_0^r|v_t|^2dt<\infty$.
Consider the stochastic optimal control problem
\begin{equation}
\label{eq:control-problem-detailed}
V(r,\nu_r)
:=
\inf_{(\mathbb Q,v)\in\mathcal A_r}
\mathbb E^{\mathbb Q}
\left[
\frac12\int_0^r|v_t|^2dt
-
\log z_r
\right],
\end{equation}
where the objective is interpreted in the extended sense.  Define
\[
Z_t^{\nu,r}(X)
:=
\mathbb E^{\mathbb P}
\left[
z_r\mid\mathcal F_t(X)
\right],
\qquad 0\leq t\leq r.
\]
Then, $Z_0^{\nu,r}(X)=1$, $Z_r^{\nu,r}(X)=z_r$, and there exists a unique
$\xi^{\nu,r}\in
L_a^2([0,r],\mathcal F(X),\mathbb P;\mathbb R^d)$ such that
\[
Z_t^{\nu,r}(X)
=
1+\int_0^t\langle\xi_s^{\nu,r},dB_s\rangle,
\qquad
P_t(X)\xi_t^{\nu,r}=\xi_t^{\nu,r}
\quad dt\otimes d\mathbb P\text{-a.e.}
\]
Set
\[
u_t^{\nu,r}
:=
\frac{\xi_t^{\nu,r}}{Z_t^{\nu,r}(X)}
\quad\text{on }\{Z_t^{\nu,r}(X)>0\},
\qquad
u_t^{\nu,r}:=0
\quad\text{otherwise},
\]
and define
\[
\frac{d(\mathbb Q^{\nu,r}|_{\mathcal G_r})}
     {d(\mathbb P|_{\mathcal G_r})}
=z_r(X_r).
\]
{Then we have $(\mathbb Q^{\nu,r},u^{\nu,r})\in\mathcal A_r$, such that}
$P_t(X)u_t^{\nu,r}=u_t^{\nu,r}$
$dt\otimes d\mathbb Q^{\nu,r}$-a.e., and
\[
B_t^{\nu,r}:=
B_t-\int_0^t u_s^{\nu,r}\,ds
\]
is a $(\mathcal G_t,\mathbb Q^{\nu,r})$-Brownian motion.  Under this
measure, the coordinate process satisfies
\begin{equation}
\label{eq:optimally-controlled}
dX_t
=
\left(
b(t,X_t)+\sigma(t,X_t)u_t^{\nu,r}
\right)dt
+
\sigma(t,X_t)dB_t^{\nu,r}.
\end{equation}
Moreover,
\[
\mathcal L_{\mathbb Q^{\nu,r}}(X_r)=\nu_r,
\qquad
V(r,\nu_r)=0,
\]
{and $u^{\nu,r}$ is the unique minimum-energy Girsanov shift realizing}
$\nu_r$:
\[
\frac12
\mathbb E^{\mathbb Q^{\nu,r}}
\int_0^r|u_t^{\nu,r}|^2dt
=
\inf_{\substack{(\mathbb Q,v)\in\mathcal A_r:\\
\mathcal L_{\mathbb Q}(X_r)=\nu_r}}
\frac12
\mathbb E^{\mathbb Q}
\int_0^r|v_t|^2dt.
\]

In particular, under Assumption~\ref{ass:2}, one may take
$\nu_r=\mathcal L_{\mathbb P}(X_r\mid Y=y)$.  We then write
$z_r^y$, $Z^{y,r}(X)$, $u^{y,r}$, and $\mathbb Q^{y,r}$ for the corresponding
objects.
\end{theorem}

\begin{remark}
{Pairs induced by strictly positive true stochastic exponentials (a.k.a. Dol\'eans-Dade exponentials) are special}
cases of $\mathcal A_r$.  Allowing $\mathbb Q$ to be merely absolutely
{continuous, rather than equivalent, includes corresponding densities that may}
vanish under $\mathbb P$.
\end{remark}

\label{app:proof-degenerate-stochastic-control}

\begin{proof}[Proof of Theorem~\ref{thrm:10}]
Since $z_r\in L^2(\mathcal F_r(X),\mathbb P)$ and $X_0$ is
deterministic,
\[
Z_0^{\nu,r}(X)=1,
\qquad
Z_r^{\nu,r}(X)=z_r.
\]
Theorem~\ref{thrm:rep} implies that there exists $\mathcal F(X)$-adapted square integrable $\xi^{\nu,r}$ such that
\[
Z_t^{\nu,r}(X) =  1+\int_0^t \langle\xi_s^{\nu,r},dB_s\rangle, 
\qquad 
P_s(X)\xi_s^{\nu,r}=\xi_s^{\nu,r}.
\]
Define
\[
u_s^{\nu,r}
:=
\frac{\xi_s^{\nu,r}}{Z_s^{\nu,r}(X)}
\mathbf 1_{\{Z_s^{\nu,r}(X)>0\}}.
\]
Then
\[
dZ_t^{\nu,r}(X) = Z_t^{\nu,r}(X)\langle u_t^{\nu,r},dB_t\rangle,
\qquad
P_t(X)u_t^{\nu,r}=u_t^{\nu,r}.
\]
Moreover,
\[
\{Z_t^{\nu,r}(X)=0\}
\]
is $dt\otimes d\mathbb Q^{\nu,r}$-null.

The localization and uniform-integrability
argument in the proof of Theorem~\ref{thm:main} applies with terminal density
$z_r(X_r)$.  It shows that $B^{\nu,r}$ is Brownian under
$\mathbb Q^{\nu,r}$ and that
\[
\frac12\mathbb E^{\mathbb Q^{\nu,r}}
\int_0^r|u_t^{\nu,r}|^2dt
=
D_{\mathrm{KL}}(\mathbb Q^{\nu,r}\Vert\mathbb P).
\]
Thus $(\mathbb Q^{\nu,r},u^{\nu,r})\in\mathcal A_r$.  Since
$d\mathbb Q^{\nu,r}/d\mathbb P=z_r(X_r)$,
\[
\mathcal L_{\mathbb Q^{\nu,r}}(X_r)=\nu_r.
\]
The original SDE identity holds $\mathbb Q^{\nu,r}$-a.s.; substituting
$dB_t=dB_t^{\nu,r}+u_t^{\nu,r}\,dt$ proves
\eqref{eq:optimally-controlled}.

For any $(\mathbb Q,v)\in\mathcal A_r$, let
$\lambda_r:=\mathcal L_{\mathbb Q}(X_r)$.  Data processing inequality for relative
entropy gives
\[
D_{\mathrm{KL}}(\mathbb Q\Vert\mathbb P)
\geq D_{\mathrm{KL}}(\lambda_r\Vert\mu_r).
\]
Therefore, with the standard extended-value convention,
\begin{align*}
\mathbb E^{\mathbb Q}
\left[
\frac12\int_0^r|v_t|^2dt-\log z_r(X_r)
\right]
&=
D_{\mathrm{KL}}(\mathbb Q\Vert\mathbb P)
-\int\log z_r\,d\lambda_r\\
&\geq
D_{\mathrm{KL}}(\lambda_r\Vert\nu_r)
\geq0.
\end{align*}
{The corresponding pair $(\mathbb Q^{\nu,r},u^{\nu,r})$ attains zero, and hence $V(r,\nu_r)=0$.}

Now restrict to competitors satisfying $\lambda_r=\nu_r$.  Then
\[
\frac12\mathbb E^{\mathbb Q}\int_0^r|v_t|^2dt
\geq
D_{\mathrm{KL}}(\nu_r\Vert\mu_r)
=
\frac12
\mathbb E^{\mathbb Q^{\nu,r}}
\int_0^r|u_t^{\nu,r}|^2dt.
\]
If equality holds, equality also holds in data processing.  Its equality
case shows that $d\mathbb Q/d\mathbb P$ is $\sigma(X_r)$-measurable.
Since its pushforward density is $z_r$, necessarily
\[
\frac{d(\mathbb Q|_{\mathcal G_r})}
     {d(\mathbb P|_{\mathcal G_r})}
=
z_r
=
\frac{d(\mathbb Q^{\nu,r}|_{\mathcal G_r})}
     {d(\mathbb P|_{\mathcal G_r})}.
\]
Thus $\mathbb Q=\mathbb Q^{\nu,r}$.  Under this common measure, both
$B-\int_0^\cdot v_s\,ds$ and
$B-\int_0^\cdot u_s^{\nu,r}\,ds$ are Brownian motions.  Uniqueness of the
continuous semimartingale decomposition gives
$v_t=u_t^{\nu,r}$ $dt\otimes d\mathbb Q^{\nu,r}$-a.e.

Note that, for any bounded Borel function $\varphi$,
\begin{align*}
\mathbb E^{\mathbb P}
\left[
\varphi(X_r)
\mathbb E^{\mathbb P}\!\left[Z_r^y(X)\mid X_r\right]
\right]
&=
\mathbb E^{\mathbb P}
\left[
\varphi(X_r)Z_r^y(X)
\right] \\
&=
\mathbb E^{\mathbb P^y}
\left[
\varphi(X_r)
\right] \\
&=
\int \varphi(x) z_r^y(x)\,
\mathcal L_{\mathbb P}(X_r)(dx).
\end{align*}
Thus, $\mathcal L_{\mathbb P}(X_r\mid Y=y)\ll\mu_r$, and
\[
z_r^y(X_r)
=
\mathbb E^{\mathbb P}
[Z_r^{y}(X)\mid X_r].
\]
The $L^2$-contraction and Assumption~\ref{ass:2} give
$z_r^y\in L^2(\mathbb P)$, and thus all preceding conclusions apply.
\end{proof}

\section{Proofs of the degenerate Tweedie formulae}\label{app:e}

\subsection{Proof of Degenerate Tweedie formula-1}

\begin{proof}[Proof of Theorem~\ref{thrm:quaso}]
Apply Theorem~\ref{thm:main}.  It gives \eqref{eq:control} in the stated
$L^2$ space with $u_t(X_t,y)=\xi_t/Z_t(X)$.  Since
$P_t(X)\xi_t=\xi_t$ $dt\otimes d\mathbb P$-a.e., we also have
$P_t(X)u_t(X_t,y)=u_t(X_t,y)$
$dt\otimes d\mathbb P^y$-a.e., and
uniqueness follows from the uniqueness of the integrand in the martingale representation by
Theorem~\ref{thrm:rep}.  The conditional-law and minimum-entropy claims on every $[0,r]$, $r<T$, follow directly from Theorem~\ref{thrm:causal}.
\end{proof}

\subsection{Proof of Degenerate Tweedie formula-2}

\begin{proof}[Proof of Corollary~\ref{cor:sd}]
Fix $r<T$, $0\leq t\leq r$, and
$0<\varepsilon<T-r$.
We apply Theorem~\ref{thrm} componentwise to
$B_{t+\varepsilon}-B_t$.

For each $j\in\{1,\ldots,d\}$, $j^{\mathrm{th}}$-component of $B_{t+\varepsilon}-B_t$ satisfies
\[
(B_{t+\varepsilon}-B_t)^j
=
\int_0^T
\left\langle
\mathbf 1_{(t,t+\varepsilon]}(s)e_j,dB_s
\right\rangle,
\]
and
$ \mathbb E^{\mathbb P} \left[ (B_{t+\varepsilon}-B_t)^j\mid X_0 \right] =0, $
since $B_{t+\varepsilon}-B_t$ is independent of
$\mathcal G_t$, and hence of $\sigma(X_0)$.
Therefore, Theorem~\ref{thrm} gives
\[
\mathbb E^{\mathbb P}
\left[
(B_{t+\varepsilon}-B_t)^j
\mid
\mathcal F_T(X)
\right]
=
\int_t^{t+\varepsilon}
\left\langle
P_s(X)e_j,dB_s
\right\rangle.
\]
Applying this componentwise yields
\begin{equation}\label{eq:raw-projected-intertwining}
\mathbb E^{\mathbb P}
\left[
B_{t+\varepsilon}-B_t
\mid
\mathcal F_T(X)
\right]
=
\int_t^{t+\varepsilon}P_s(X)\,dB_s.
\end{equation}

Since $Y=G(X_T)$,
$\sigma\{X_t,Y\}\subseteq\mathcal F_T(X)$.
Hence, by the tower property and
\eqref{eq:raw-projected-intertwining},
\begin{align*}
\mathbb E^{\mathbb P}
\left[
B_{t+\varepsilon}-B_t
\mid X_t,Y
\right]
&=
\mathbb E^{\mathbb P}
\left[
\left.
\mathbb E^{\mathbb P}
\left[
B_{t+\varepsilon}-B_t
\mid\mathcal F_T(X)
\right]
\right|
X_t,Y
\right]
\\
&=
\mathbb E^{\mathbb P}
\left[
\left.
\int_t^{t+\varepsilon}P_s(X)\,dB_s
\right|
X_t,Y
\right].
\end{align*}
Thus, the result follows from Theorem~\ref{thrm:quaso}.
\end{proof}

\subsection{Proof of Degenerate Tweedie formula-3}

\begin{proof}[Proof of Theorem~\ref{thm:state-tweedie}]
Fix $R$ with $r<R<T$.  Under the corresponding measure
$\mathbb Q^{y,R}$ of Theorem~\ref{thrm:causal}, for
$0<\varepsilon<R-r$ and $0\leq t\leq r$,
\begin{align*}
X_{t+\varepsilon}-X_t
&=\int_t^{t+\varepsilon}
  \bigl(b(s,X_s)+\sigma(s,X_s)u_s(X_s,y)\bigr)ds
  +\int_t^{t+\varepsilon}\sigma(s,X_s)\,dB_s^{y,R}.
\end{align*}
The stochastic integral is a square-integrable
$(\mathcal G_t,\mathbb Q^{y,R})$-martingale increment, so its conditional
expectation given $X_t$ is zero.  Since $\mathbb Q^{y,R}$ and
$\mathbb P^y$ agree on $\mathcal F_R(X)$, this gives
\begin{align*}
\mathbb E^{\mathbb P^y}\!\left[
\left.\frac{X_{t+\varepsilon}-X_t}{\varepsilon}
-b(t,X_t)\right|X_t\right] 
&=\mathbb E^{\mathbb P^y}\!\left[
\left.\frac1\varepsilon\int_t^{t+\varepsilon}
\bigl(b(s,X_s)-b(t,X_t)\bigr)\,ds
\right|X_t\right] \\
&\quad+\mathbb E^{\mathbb P^y}\!\left[
\left.\frac1\varepsilon\int_t^{t+\varepsilon}
\sigma(s,X_s)u_s(X_s,y)\,ds
\right|X_t\right].
\end{align*}
By Assumption~\ref{ass:1} and Theorem~\ref{thrm:quaso},
\[
\mathbb E^{\mathbb P^y}\int_0^r
|\sigma(t,X_t)u_t(X_t,y)|^2dt
\leq \|\sigma\|_\infty^2
\mathbb E^{\mathbb P^y}\int_0^r|u_t(X_t,y)|^2dt<\infty.
\]
Thus, by Assumption \ref{ass:1}, the bound above, and the $L^2$-Lebesgue differentiation theorem,
\begin{align*}
\frac1\varepsilon\int_t^{t+\varepsilon}
\bigl(b(s,X_s)-b(t,X_t)\bigr)\,ds&\longrightarrow0,\\
\frac1\varepsilon\int_t^{t+\varepsilon}
\sigma(s,X_s)u_s(X_s,y)\,ds
&\longrightarrow\sigma(t,X_t)u_t(X_t,y)
\end{align*}
in $L^2([0,r]\times\Omega,dt\otimes d\mathbb P^y)$. For any such pair
$H_t^\varepsilon\to H_t$,
\[
\int_0^r\mathbb E^{\mathbb P^y}
\left|\mathbb E^{\mathbb P^y}
 [H_t^\varepsilon\mid X_t]-H_t\right|^2dt
\leq
\int_0^r\mathbb E^{\mathbb P^y}|H_t^\varepsilon-H_t|^2dt
\longrightarrow0.
\]
Thus both convergences remain valid after conditioning on $X_t$.
Passing to the limit in the preceding identity proves
\eqref{eq:state-tweedie}.
\end{proof}

\section{Terminal-Increment Regression for Additive Diffusions}\label{app:f}

For additive diffusions with a fixed noise matrix, the local Brownian
target can be replaced exactly by a terminal-increment target.  This removes
the auxiliary limit $\varepsilon\downarrow0$ required in the the definition of the loss function.

\begin{theorem}\label{thm:terminal-bridge-loss}
Let $B$ be a $d$-dimensional $(\mathcal G_t)$-Brownian motion, let $X_0$
be $\mathcal G_0$-measurable, and suppose that
\[
X_t=X_0+\int_0^t b(s)\,ds+\sigma B_t,
\]
where $b:[0,T]\to\mathbb R^n$ is deterministic and integrable and
$\sigma\in\mathbb R^{n\times d}$ is fixed, possibly rank deficient.  Let
$Y=G(X_T)$ for a Borel map $G$.
Fix $r<T$, let $\mathscr H_r$ be the $\sigma$-field on
$[0,r]\times\Omega$ generated by
$(t,\omega)\mapsto(t,X_t(\omega),Y(\omega))$.
For $Q_t:=(B_T-B_t)/(T-t)$, define the product-space conditional
expectation
\[
\bar u:=\mathbb E^{dt\otimes d\mathbb P}[Q\mid\mathscr H_r].
\]
Choose a Borel representative
$\bar u:[0,r]\times\mathbb R^n\times\mathsf{Y}\to\mathbb R^d$ for this
conditional expectation.  Then,
$\bar u(t,X_t,Y)=\mathbb E^{\mathbb P}[Q_t\mid X_t,Y]$
$dt\otimes d\mathbb P$-a.e.
For every $0<\varepsilon<T-r$,
\begin{equation}\label{eq:terminal-local-equivalence}
\bar u(t,X_t,Y)
=\mathbb E^{\mathbb P}\left[
\left.\frac{P(B_T-B_t)}{T-t}\right|X_t,Y\right]
=\mathbb E^{\mathbb P}\left[
\left.\frac{B_{t+\varepsilon}-B_t}{\varepsilon}\right|X_t,Y\right]
\end{equation}
for $dt\otimes d\mathbb P$-a.e. $(t,\omega)\in[0,r]\times\Omega$.
Moreover,
\[
\mathbb E^{\mathbb P}\int_0^r
\left|\frac{B_T-B_t}{T-t}\right|^2dt
=d\log\frac{T}{T-r}<\infty,
\]
and the regression problem
\begin{equation}\label{eq:terminal-bridge-regression}
\inf_{v\in L^2(\mathscr H_r,dt\otimes d\mathbb P;\mathbb R^d)}
\mathbb E^{\mathbb P}\int_0^r
\left|v_t-\frac{B_T-B_t}{T-t}\right|^2dt
\end{equation}
has the unique minimizer $\bar u$.

If Assumptions~\ref{ass:1}--\ref{ass:2} also hold, let $\xi^y$ be the
 integrand from Lemma~\ref{lem:density-martingale}.  For
$\mathcal L_{\mathbb P}(Y)$-a.e.\ $y$, let
$u_t^{\mathrm{path},y}:=\xi_t^y/Z_t^y(X)$ on $\{Z_t^y(X)>0\}$ and set it
to zero otherwise.  Then
$\bar u(t,X_t,y)=u_t^{\mathrm{path},y}$ in
$L^2([0,r]\times\Omega,dt\otimes d\mathbb P^y)$.  Consequently, the
path-adapted control admits this state-feedback version, and the
terminal-increment loss recovers the corresponding controlled state law on
$[0,r]$ and its minimum-energy property.
\end{theorem}

\begin{proof}
Fix $0\leq t\leq r$ and $0<\varepsilon<T-r$, and write
\[
\Delta B_{t,T}:=B_T-B_t,
\qquad
\Delta B_{t,\varepsilon}:=B_{t+\varepsilon}-B_t.
\]
The Brownian bridge decomposition is
\begin{equation}\label{eq:brownian-bridge-decomposition}
\Delta B_{t,\varepsilon}
=\frac{\varepsilon}{T-t}\Delta B_{t,T}+R_{t,\varepsilon},
\end{equation}
where $R_{t,\varepsilon}$ is centered Gaussian and independent of
$\mathcal G_t\vee\sigma(\Delta B_{t,T})$.  Since
\[
X_T=X_t+\int_t^T b(s)\,ds+\sigma\Delta B_{t,T},
\]
the pair $(X_t,Y)$ is measurable with respect to
$\mathcal G_t\vee\sigma(\Delta B_{t,T})$.  Hence
$\mathbb E^{\mathbb P}[R_{t,\varepsilon}\mid X_t,Y]=0$, and conditioning
\eqref{eq:brownian-bridge-decomposition} gives
\[
\mathbb E^{\mathbb P}\left[
\left.\frac{B_{t+\varepsilon}-B_t}{\varepsilon}\right|X_t,Y\right]
=\mathbb E^{\mathbb P}\left[
\left.\frac{B_T-B_t}{T-t}\right|X_t,Y\right].
\]

Since $P$ is the orthogonal projection onto
$\operatorname{Ran}(\sigma^\top)$, we have $\sigma(I-P)=0$.  The Gaussian
vectors $P\Delta B_{t,T}$ and $(I-P)\Delta B_{t,T}$ are independent, and
the latter is jointly independent of
$\mathcal G_t\vee\sigma(P\Delta B_{t,T})$.  Since $Y$ depends on the
future increment only through $P\Delta B_{t,T}$,
\[
\mathbb E^{\mathbb P}[(I-P)\Delta B_{t,T}\mid X_t,Y]=0.
\]
This proves all identities in \eqref{eq:terminal-local-equivalence}.

The stated integrability follows from
\[
\mathbb E^{\mathbb P}|B_T-B_t|^2=d(T-t).
\]
The regression claim is the orthogonal-projection characterization of
conditional expectation in the Hilbert space $\mathcal V_r$.  Finally,
choose a deterministic sequence $\varepsilon_k\downarrow0$ and disintegrate
the exact identities in \eqref{eq:terminal-local-equivalence}
simultaneously outside one $\mathcal L_{\mathbb P}(Y)$-null set.  The same
independence argument with conditioning on
$\mathcal F_t(X)\vee\sigma(Y)$ gives
\[
\mathbb E^{\mathbb P}\left[
\left.\frac{B_{t+\varepsilon_k}-B_t}{\varepsilon_k}
\right|\mathcal F_t(X)\vee\sigma(Y)\right]
=\mathbb E^{\mathbb P}\left[
\left.\frac{B_T-B_t}{T-t}\right|X_t,Y\right].
\]
The path-feedback form of Corollary~\ref{cor:sd} therefore yields
$\bar u(t,X_t,y)=u_t^{\mathrm{path},y}$ in the stated fiberwise $L^2$
space.  In particular, the
additive model supplies its own state-feedback version without
Assumption~\ref{ass:feedback}.  The controlled-law and minimum-energy
conclusions follow from the path-feedback form of
Theorem~\ref{thrm:causal}.
\end{proof}

\section{Implementation of the Training Procedure}
\label{app:training-consistency}

This section studies the procedure for the convergence to the
control $u$ appearing in \eqref{eq:control}. We consider both the
projected target used in Proposition~\ref{prop} and the unprojected
Brownian-increment target from Corollary~\ref{cor:sd}.

\subsection{Population regression}

Assume throughout this subsection that Assumptions~\ref{ass:1},
\ref{ass:2}, and \ref{ass:feedback} hold.
Fix $r<T$ and $0<\varepsilon<T-r$. Define
\[
M_{t}^{\varepsilon, \mathrm{proj}}
:=
\frac{1}{\varepsilon}
\int_t^{t+\varepsilon}P_s(X)\,dB_s
\text{ and }
M_{t}^{\varepsilon,\mathrm{raw}}
:=
\frac{B_{t+\varepsilon}-B_t}{\varepsilon}.
\]
Let $\mathscr H_r$ be the product $\sigma$-field on
$[0,r]\times\Omega$ generated by the map
$(t,\omega)\mapsto(t,X_t(\omega),Y(\omega))$, and let
\[
\mathcal V_r
:=L^2(\mathscr H_r,dt\otimes d\mathbb P;\mathbb R^d).
\]
This is the closed regression space corresponding to state feedback
controls that appears in \eqref{eq:s2}.

For either choice
$M_{t}^{\varepsilon}\in \{M_{t}^{\varepsilon,\mathrm{proj}},
M_{t}^{\varepsilon,\mathrm{raw}}\}$,
consider the population loss
\[
\mathcal L_\varepsilon(v)
=
\mathbb E^{\mathbb P}
\int_0^r
\left|v_t-M_{t}^{\varepsilon}\right|^2dt.
\]
Its unique minimizer over $\mathcal V_r$ is the product-space conditional
expectation.  Choose a Borel representative
$u_\varepsilon^*:[0,r]\times\mathbb R^n\times\mathsf{Y}\to\mathbb R^d$;
then, $dt\otimes d\mathbb P$-a.e.,
\[
u_\varepsilon^*(t,X_t,Y)
=
\mathbb E^{\mathbb P}
\left[
M_{t}^{\varepsilon}\mid X_t,Y
\right].
\]
When no arguments are displayed, $u_\varepsilon^*$ denotes the realized
process $u_\varepsilon^*(t,X_t,Y)$.

Indeed, the orthogonality property of conditional expectation gives
\begin{equation}
\label{eq:population-excess-risk}
\mathcal L_\varepsilon(v)
-
\mathcal L_\varepsilon(u_\varepsilon^*)
=
\mathbb E^{\mathbb P}
\int_0^r
\left|v_t-u_\varepsilon^*(t,X_t,Y)\right|^2dt.
\end{equation}
Consequently, minimizing the squared regression loss is equivalent to
approximating the appropriate conditional expectation.

After disintegration with respect to $Y$,
Proposition~\ref{prop} and Corollary~\ref{cor:sd} give, for
$\mathcal L_{\mathbb P}(Y)$-a.e.\ $y$,
\[
u_{\varepsilon}^{*,\mathrm{proj}}
\longrightarrow u
\quad\text{in }L^2([0,r]\times\Omega,dt\otimes d\mathbb P^y),
\qquad
u_{\varepsilon}^{*,\mathrm{raw}}
\longrightarrow u
\quad\text{in }L^2([0,r]\times\Omega,dt\otimes d\mathbb P^y).
\]
Let $g_\varepsilon(y)$ be such that 
\[
\mathbb E^{\mathbb P}\int_0^r
\left|
u_\varepsilon^*(t,X_t,Y)
-
u_t(X_t,Y)
\right|^2dt
=
\int g_\varepsilon(y)\,
\mathcal L_{\mathbb P}(Y)(dy).
\]
We additionally assume convergence in
$L^2([0,r]\times\Omega,dt\otimes d\mathbb P)$; it follows, for example,
if $g_\varepsilon(y)\to0$ for
$\mathcal L_{\mathbb P}(Y)$-a.e.\ $y$ and the family
$(g_\varepsilon)_\varepsilon$ is uniformly integrable under
$\mathcal L_{\mathbb P}(Y)$.  Under this joint regularity conditions, both losses recover the same minimum-energy control.

\subsection{Discretization of the projected target}

In this subsection, we provide an approximation scheme for $M_{t_j}^{\varepsilon,\mathrm{proj}}$ that appears in \eqref{eq:proj}, whose convergence is not obvious since the projection $(P_t(X))_t$ is not continuous.

Let $t_k=kh$ be a uniform grid, assume $T=Kh$ for an integer $K$, and
suppose that $\varepsilon=qh$ for an integer $q\geq1$.  Let $X^h$ be the Euler--Maruyama
approximation of $X$, and let $\eta_h(0)=0$ and $\eta_h(s)=t_k$ for
$s\in(t_k,t_{k+1}]$.  Define the predictable piecewise-constant
approximation
\[
\overline P_s^h
:=
P_{\eta_h(s)}(X^h),
\qquad 0\leq s\leq T,
\]
where $\overline P_0^h=P_0(X^h)$.  Equivalently,
\[
\overline P_s^h=P_{t_k}(X^h),
\qquad s\in(t_k,t_{k+1}].
\]

For $0\leq t\leq r$, set
\[
M_t^{\varepsilon,\mathrm{proj},h}
:=
\frac{1}{\varepsilon}
\int_t^{t+\varepsilon}\overline P_s^h\,dB_s \implies M_{t_j}^{\varepsilon,\mathrm{proj},h}
=
\frac{1}{qh}
\sum_{k=j}^{j+q-1}
P_{t_k}(X^h)
\bigl(B_{t_{k+1}}-B_{t_k}\bigr).
\]

By It\^o's isometry and Fubini's theorem,
\begin{align}
&\mathbb E^{\mathbb P}
\int_0^r
\left|
M_t^{\varepsilon,\mathrm{proj},h}
-
M_t^{\varepsilon,\mathrm{proj}}
\right|^2dt
\leq
\frac{1}{\varepsilon}
\mathbb E^{\mathbb P}
\int_0^{r+\varepsilon}
\left\|
\overline P_s^h-P_s(X)
\right\|_{\mathrm F}^2ds,
\label{eq:projected-target-error}
\end{align}
where $\|\cdot\|_{\mathrm F}$ is the Frobenius norm.
Therefore, for fixed $\varepsilon$, the discrete target converges
whenever
\[
\mathbb E^{\mathbb P}
\int_0^{r+\varepsilon}
\left\|
\overline P_s^h-P_s(X)
\right\|_{\mathrm F}^2ds
\longrightarrow0.
\]

For example, suppose that the state-dependent projection $P(t,x)$ satisfies
\begin{equation}\label{proj}
\|P(t,x)-P(s,z)\|_{\mathrm F}
\leq
L_P\bigl(|x-z|+|t-s|^{1/2}\bigr),
\end{equation}
and that the numerical approximation of $X$ has mean-square error
of order $h$. Then,
\[
\mathbb E^{\mathbb P}
\int_0^{r+\varepsilon}
\left\|
\overline P_s^h-P_s(X)
\right\|_{\mathrm F}^2ds
\leq Ch
\implies
\mathbb E^{\mathbb P}
\int_0^r
\left|
M_t^{\varepsilon,\mathrm{proj},h}
-
M_t^{\varepsilon,\mathrm{proj}}
\right|^2dt
\leq C\frac{h}{\varepsilon}.
\]
The displayed bound is driven to zero by the sufficient conditions
$\varepsilon\downarrow0$ and $h/\varepsilon\to0$.

Below, we provide a sufficient condition for \eqref{proj}.

\begin{proposition}
\label{prop:projected-target-rate}
Let $X^h$ be the Euler--Maruyama approximation of $X$.  Suppose that
$b$ and $\sigma$ are globally Lipschitz in the state
variable, $1/2$-H\"older continuous in time, and of linear growth, with
$\mathbb E^{\mathbb P}|X_0|^2<\infty$. Assume further that, for some
$\rho\geq1$ and $\kappa>0$,
\[
\operatorname{rank}\sigma(t,x)=\rho,
\qquad
s_\rho\bigl(\sigma(t,x)\bigr)\geq\kappa
\]
for every $(t,x)$, where $s_\rho(\sigma)$ denotes the smallest nonzero
singular value of $\sigma$. Then, there is a constant $C$, independent
of $h$ and $\varepsilon$, such that
\[
\mathbb E^{\mathbb P}
\int_0^{r+\varepsilon}
\left\|
\overline P_s^h-P_s(X)
\right\|_{\mathrm F}^2ds
\leq Ch.
\]
Consequently,
\[
\mathbb E^{\mathbb P}\int_0^r
\left|M_t^{\varepsilon,\mathrm{proj},h}
-M_t^{\varepsilon,\mathrm{proj}}\right|^2dt
\leq C\frac{h}{\varepsilon}.
\]
\end{proposition}

\begin{proof}
The constant-rank assumption and the uniform lower bound on the smallest nonzero singular value imply that the orthogonal projection $\Pi(a)=a^\dagger a$ onto $\operatorname{Ran}(a^\top)$ is Lipschitz continuous as a function of a on this class of matrices. For \(a,\widetilde a\in\mathbb R^{n\times d}\) satisfying
\[
\operatorname{rank}(a)=\operatorname{rank}(\widetilde a)=\rho,
\qquad
s_\rho(a),\,s_\rho(\widetilde a)\ge \kappa,
\]
let
\[
\Pi(a):=a^\dagger a,
\qquad
\Pi(\widetilde a):=\widetilde a^\dagger\widetilde a
\]
denote the orthogonal projections onto
\(\operatorname{Ran}(a^\top)\) and
\(\operatorname{Ran}(\widetilde a^\top)\), respectively.
Since
\[
\widetilde a\bigl(I-\Pi(\widetilde a)\bigr)=0,
\]
we have
\[
\Pi(a)\bigl(I-\Pi(\widetilde a)\bigr)
=
a^\dagger(a-\widetilde a)
\bigl(I-\Pi(\widetilde a)\bigr).
\]
 Moreover, \[ \|a^\dagger\|_{\mathrm{op}} = \frac{1}{s_\rho(a)} \leq \frac1\kappa, \] and therefore \[ \|\Pi(a)(I-\Pi(\widetilde a))\|_{\mathrm F} \leq \frac1\kappa\|a-\widetilde a\|_{\mathrm F} \implies \|(I-\Pi(a))\Pi(\widetilde a)\|_{\mathrm F} \leq \frac1\kappa\|a-\widetilde a\|_{\mathrm F}. \] 

Since \[ \Pi(a)-\Pi(\widetilde a) = \Pi(a)(I-\Pi(\widetilde a)) - (I-\Pi(a))\Pi(\widetilde a), \] we obtain \[ \|\Pi(a)-\Pi(\widetilde a)\|_{\mathrm F} \leq \frac{2}{\kappa}\|a-\widetilde a\|_{\mathrm F} \implies \|\widetilde \Pi(t,\sigma(\cdot,x))-\widetilde \Pi(s,\sigma(\cdot,z))\|_{\mathrm F} \leq \frac{2}{\kappa} \|\sigma(t,x)-\sigma(s,z)\|_{\mathrm F}.\]  Using the Lipschitz continuity of $\sigma$ in the state variable and its $1/2$-H\"older continuity in time, we conclude that \[ \|\widetilde \Pi(t,\sigma(\cdot,x))-\widetilde \Pi(s,\sigma(\cdot,z))\|_{\mathrm F} \leq C\bigl(|x-z|+|t-s|^{1/2}\bigr), \] where $C$ depends only on $\kappa$ and the corresponding regularity constants of $\sigma$.

The standard strong Euler--Maruyama estimate and the increment estimate
for $X$ yield
\[
\mathbb E^{\mathbb P}
\left|
X_s-X_{\eta_h(s)}^h
\right|^2
\leq Ch,
\]
see \cite[Theorem 10.2.2.]{MR1214374}.
Hence,
\begin{align*}
\mathbb E^{\mathbb P}
\left\|
\overline P_s^h-P_s(X)
\right\|_{\mathrm F}^2
&\leq
C\left(
\mathbb E^{\mathbb P}
\left|X_{\eta_h(s)}^h-X_s\right|^2
+
|s-\eta_h(s)|
\right) \\
&\leq Ch.
\end{align*}
Integrating over $[0,r+\varepsilon]$ proves the first assertion.

Finally, It\^o's isometry and Fubini's theorem give
\begin{align*}
\mathbb E^{\mathbb P}
\int_0^r
\left|
M_t^{\varepsilon,\mathrm{proj},h}
-
M_t^{\varepsilon,\mathrm{proj}}
\right|^2dt
&=
\frac1{\varepsilon^2}
\int_0^r
\mathbb E^{\mathbb P}
\int_t^{t+\varepsilon}
\left\|
\overline P_s^h-P_s(X)
\right\|_{\mathrm F}^2ds\,dt \\
&\leq
\frac1\varepsilon
\mathbb E^{\mathbb P}
\int_0^{r+\varepsilon}
\left\|
\overline P_s^h-P_s(X)
\right\|_{\mathrm F}^2ds \\
&\leq C\frac{h}{\varepsilon}.
\end{align*}
\end{proof}
\subsection{The unprojected loss function}

When Corollary~\ref{cor:sd} is used, the discrete target is
\begin{equation}\label{unprojected}
M_{t_j}^{\varepsilon,\mathrm{raw}} = \frac{B_{t_{j+q}}-B_{t_j}}{qh} = \frac{1}{qh} \sum_{k=j}^{j+q-1} \bigl(B_{t_{k+1}}-B_{t_k}\bigr). 
\end{equation}
This equality is exact. Hence, unlike the projected case, there is no
stochastic-integral discretization error and no regularity assumption
on $P_t(X)$ is needed. A time discretization is still required for the
integral in the loss and, when $X$ is simulated numerically, for the
state process itself.

\subsection{Monte Carlo and optimization consistency}

To approximate the loss functions we defined for the Tweedie's formula, a standard approach is to use neural networks. We refer to \cite[Section 3]{tzen2019theoretical} for a theoretical study of the feasibility of this approach. Their analysis is directly applicable in our setting.

\end{document}